\documentclass[]{article}

\usepackage{hyperref}
\usepackage{url}

\usepackage{graphicx} 
\usepackage{placeins}
\usepackage{caption}  
\usepackage{tikz}
\usepackage{amsmath}
\usepackage{amsfonts}
\usepackage{amsthm, mathrsfs, mathtools}
\usepackage{MnSymbol}
\usepackage{algorithm}
\usepackage{algorithmicx}
\usepackage[noend]{algpseudocode}
\algrenewcommand\algorithmicindent{1em}
\newenvironment{protocol}[1][htb]
{\floatname{algorithm}{Online Protocol}%
	\begin{algorithm}[#1]}
	{\end{algorithm}%
	\floatname{algorithm}{Algorithm}}
\usepackage{aligned-overset} 
\usepackage{enumitem}
\usepackage[capitalize,noabbrev]{cleveref} 

\newcommand{\cC}{\mathcal{C}}
\newcommand{\cD}{\mathcal{D}}
\newcommand{\cE}{\mathcal{E}}
\newcommand{\cF}{\mathcal{F}}

\newcommand{\cI}{\mathcal{I}}

\newcommand{\cS}{\mathcal{S}}

\newcommand{\E}{\mathbb{E}}

\newcommand{\I}{\mathbb{I}}

\newcommand{\N}{\mathbb{N}}

\newcommand{\bbR}{\mathbb{R}}

\newcommand{\e}{\varepsilon}

\newcommand{\lrb}[1]{\left(#1\right)}
\newcommand{\brb}[1]{\bigl(#1\bigr)}

\newcommand{\lsb}[1]{\left[#1\right]}
\newcommand{\bsb}[1]{\bigl[#1\bigr]}
\newcommand{\Bsb}[1]{\Bigl[#1\Bigr]}

\newcommand{\lcb}[1]{\left\{#1\right\}}
\newcommand{\bcb}[1]{\bigl\{#1\bigr\}}

\newcommand{\labs}[1]{\left\lvert#1\right\rvert}

\newcommand{\lno}[1]{\left\lVert#1\right\rVert}

\newcommand{\dif}{\,\mathrm{d}}

\newcommand{\fracc}[2]{#1/#2}
\newcommand{\s}{\subseteq}
\newcommand{\m}{\setminus}

\newcommand{\nhphantom}[1]{\sbox0{#1}\hspace{-\the\wd0}}
\newcommand{\bisect}{\mathrm{Bisect}}
\newcommand{\length}{\mathrm{length}}

\newcommand{\Var}{\mathrm{Var}}
\newcommand{\ceq}{\coloneqq}

\DeclareSymbolFont{extraup}{U}{zavm}{m}{n}
\DeclareMathSymbol{\clubsuit}{\mathalpha}{extraup}{84}
\DeclareMathSymbol{\spadesuit}{\mathalpha}{extraup}{81}
\DeclareMathSymbol{\varheartsuit}{\mathalpha}{extraup}{86}
\DeclareMathSymbol{\vardiamondsuit}{\mathalpha}{extraup}{87}

\newcommand{\gft}{g}
\newcommand{\GFT}{\mathrm{GFT}}

\newcommand{\bx}{\boldsymbol{x}}

\newcommand{\biave}{BiAve}
\newcommand{\exbis}{ExBis}
\newcommand{\sF}{\mathscr{F}}
\newcommand{\Vtilde}{\tilde V}
\newcommand{\Wtilde}{\tilde W}
\newcommand{\J}{\mathcal{J}}

\newtheorem{theorem}{Theorem}
\newtheorem{lemma}{Lemma}

\usepackage[round]{natbib}

\title{A Tight Regret Analysis of \\
	Non-Parametric Repeated Contextual Brokerage}
\author{{\bf Fran\c{c}ois Bachoc} \\ IMT, University of Toulouse, Institut universitaire de France (IUF) \\ 
	~ \\ {\bf Tommaso Cesari} \\ EECS, University of Ottawa \\ ~ \\{\bf Roberto Colomboni} \\ 
Dept. of CS, University of Milan}

\begin{document}

\maketitle

\begin{abstract}
We study a contextual version of the repeated brokerage problem. 
In each interaction, two traders with private valuations for an item seek to buy or sell based on the learner's---a broker---proposed price, which is informed by some contextual information.
The broker's goal is to maximize the traders' net utility---also known as the gain from trade---by minimizing regret compared to an oracle with perfect knowledge of traders' valuation distributions.
We assume that traders' valuations are zero-mean perturbations of the unknown item's current market value---which can change arbitrarily from one interaction to the next---and that similar contexts will correspond to similar market prices.
We analyze two feedback settings: full-feedback, where after each interaction the traders' valuations are revealed to the broker, and limited-feedback, where only transaction attempts are revealed.
For both feedback types, we propose algorithms achieving tight regret bounds.
We further strengthen our performance guarantees by providing a tight $1/2$-approximation result showing that the oracle that knows the traders' valuation distributions achieves at least $1/2$ of the gain from trade of the omniscient oracle that knows in advance the actual realized traders' valuations.
\end{abstract}

\section{INTRODUCTION}
We investigate repeated brokerage with contextual information, where a broker (the learner) is tasked with facilitating commerce between prospective traders.
This classic setting models commerce in Over-The-Counter (OTC) markets of stock, energy, and rare minerals, to name a few, which are responsible for a massive amount of the overall world's business volume \citep{lucas1989effects,weill2020search,bis2023}.

In this problem, during each interaction $t$, two traders who own identical copies of an item for which they hold private valuations $V_t$ and $W_t$ reach out to the broker.
The traders' goal is to make a profit by trying to sell a copy of their item if the proposed price is higher than their valuation or buy a new copy if the opposite is true.
The broker observes some contextual information $\bx_t$ modeling the item in question and the market conditions and uses it along with his past knowledge to propose a trading price $P_t$ to the traders.
If the price is below one of the two traders' valuations and above the other, the trader with the highest valuation buys the item from the other trader at price $P_t$.
The broker strives to maximize the so-called \emph{gain from trade}, i.e., the sum of the net utilities gained by the traders.
Consistently with the existing literature \citep{bachoc2024contextual}, we assume that traders' valuations are independent zero-mean perturbations of market values $\mu_t$, which can change arbitrarily over time.
But unlike previous works that assume a parametric (linear) relationship between contexts and market values, we only suppose that similar contexts will correspond to similar market prices.
The goal of the broker is to minimize the \emph{regret}, defined as the loss in efficiency between the total gain from trade achieved by their strategy and the one of an idealized oracle that chooses the optimal price at each interaction given an exact knowledge of \emph{the traders' valuation distribution}.
We also consider an even more powerful oracle that has perfect knowledge of the \emph{realizations} of the traders' valuation, and discuss how the techniques we develop apply to this setting.

We study two variants of this problem: the \emph{full}-feedback setting, where the valuations of the traders are revealed to the broker after each interaction, and the \emph{limited}-feedback setting, where nothing other than the fact that the traders attempted to buy or sell is revealed to the learner after each interaction.

\subsection{Formal setting}
\label{s:setting}

We study the following online learning problem. 
%
\begin{protocol}
	\caption{Contextual Brokerage}
	\begin{algorithmic}[1]
		\State Two traders arrive with private valuations $V_{t},W_{t}$ 
		\item The broker observes a context $\bx_{t}$
		\item The broker proposes a trading price $P_t$
		\item A trade occurs iff $\min\{V_t,W_t\} \le P_t \le\max\{V_t,W_t\}$
		\item The broker observes some feedback
	\end{algorithmic}
\end{protocol}

Consistently with the existing literature, we set the reward associated with each interaction as the \emph{gain from trade}: the sum of the net utilities of the traders. 
Formally, for any $p,v,w \in [0,1]$, letting $v \vee w \ceq \max\{v,w\}$ and $v \wedge w \ceq \min\{v,w\}$, the utility of a price $p$ when the valuations of the traders are $v$ and $w$ is
\begin{align*}
	\gft(p,v,w) 
	&
	\coloneqq
	(\underbrace{v \vee w - p}_{\substack{\textrm{buyer's}\\\textrm{net gain}}} + \underbrace{p - v \wedge w}_{\substack{\textrm{seller's}\\\textrm{net gain}}} ) \I \{ \underbrace{v \wedge w \le p \le v \vee w }_{\textrm{a trade occurs}}\}
	\\
	&
	= 
	\lrb{v \vee w - v \wedge w} \I \{v \wedge w \le p \le v \vee w\}.
\end{align*}
The aim of the learner is to post prices $P_t$ (depending the history up to time $t-1$, the current context $\bx_t$ and, possibly, some internal randomization) that minimize the \emph{regret} against the best sequence of deterministic prices\footnote{Economically, this benchmark models the best choice of an oracle that knows the distributions but not the realizations of the valuations. 
	We will prove later that this is not too different from comparing against the best random prices $p_1, \dots, p_T \in [0,1]$ (that do have access to the realizations of the valuations in hindsight).
}, irrespectively of the underlying instance determining contexts and traders' valuations.
Formally: for any time horizon $T\in\N$, we define the \emph{regret} as
\begin{align*}
	R_T
	\coloneqq 
	\sup_{(\bx_t, V_t,W_t)_{t\in\N} \, \in \, \J} \Biggl( 
	\sup_{p_1, \dots, p_T \, \in \, [0,1]} \E \lsb{ \sum_{t=1}^T \gft (p_t, V_t, W_t ) } 
	- \E \lsb{ \sum_{t=1}^T \gft \brb{ P_t, V_t, W_t } } \Biggr),
\end{align*}
where the expectations are taken with respect to the randomness in $(V_t,W_t)_{t\in \N}$ and, possibly, the internal randomization used to choose the trading prices $(P_t)_{t\in \N}$, and the first supremum is over the instance set $\J$ which consists of all sequences $(\bx_t, V_t,W_t)_{t\in\N}$ of contexts and traders' valuations such that:
\begin{enumerate}
	\item \label{i:one} For all $t\in\N$, the context $\bx_t$ belongs to $[0,1)^d$.
	\item \label{i:two} There exists a sequence of \emph{market values} $\mu_1,\mu_2,\dots$ in $[0,1]$ such that, for all $t,t'\in\N$, market values $\mu_t,\mu_{t'}$, and contexts $\bx_t,\bx_{t'}$, it holds that $ \labs{\mu_t - \mu_{t'}} \le L \lno{ \bx_t - \bx_{t'} }_{\infty}$. To lighten the notation, we assume $L=1$ without loss of generality.
	\item \label{i:three} The traders' valuations $V_1,W_1, V_2,W_2, \dots$ form an independent sequence of random variables and, for all $t\in\N$, the traders' valuations $V_t$ and $W_t$ are $[0,1]$-valued random variables admitting densities upper bounded by some constant $M>0$ with a common expectation equal to the current market value $\mu_t\in[0,1]$.
\end{enumerate}
Finally, we consider the two most studied types of feedback in the bilateral trade literature. Specifically, at each round $t$, only after having posted the price $P_t$, the learner receives either:
\begin{itemize}
	\item[$\circ$] \emph{Full feedback}, i.e., the valuations $V_{t}$ and $W_{t}$ of the two current traders are disclosed.
	\item[$\circ$] \emph{Limited feedback}, i.e., only the indicator functions $\I\{P_t \le V_{t}\}$ and  $\I\{P_t \le W_{t}\}$ are disclosed.
\end{itemize}

\paragraph{Discussion on modeling assumptions.}
In \Cref{i:one}, we assume that the context space is $[0,1)^d$ merely for the sake of convenience, and without loss of generality. Our theory can be extended straightforwardly to any bounded context space at the cost of a more cumbersome notation.
In \Cref{i:two}, we assume that similar contexts relate to similar market prices. 
This natural modeling assumption quantifies the intuitive expectation that, e.g., if the broker knows that today's traders are trying to trade 99\%-pure gold, their valuations, on average, will be close to yesterday's valuations for 98\%-pure gold, and likely far from last weeks's valuation for 70\%-pure iron.
In \Cref{i:three}, we allow for fluctuations of the perceived market price from the perspective of the traders. 
Note that we do not require the sequence of valuations to be i.i.d.. 
Also note that the assumption that valuations admit a bounded density cannot be lifted, as it has been shown that even in a simplified, special case of our setting \cite[Theorem 9]{bolic2023online}, learning becomes impossible with limited feedback when this assumption is removed.

\paragraph{Discussion on feedback models.}
The information gathered in the full feedback model reflects \emph{direct revelation mechanisms}, where traders disclose their valuations $V_{t}$ and $W_{t}$ prior to each round, but the price determined by the mechanism at time $t$ is guaranteed to be based solely on the previous valuations $V_1, W_1, \dots, V_{t-1}, W_{t-1}$. 
Conversely, the limited feedback model reflects \emph{posted price} mechanisms. 
In this model, traders only indicate their willingness to buy or sell at the posted price, and their valuations $V_{t}$ and $W_{t}$ remain undisclosed.

\subsection{Our contributions}

Under the assumptions described in \Cref{s:setting}, and with the goal of designing \emph{simple} and \emph{interpretable} optimal algorithms, we make the following contributions.

\begin{enumerate}
	\item For the full-feedback setting, we design the \biave{} algorithm (\Cref{a:lipshitz:full:feedback}) and show an upper bound on its regret after $T$ interactions of order $T^{\frac{d}{d+2}}$, where $d$ is the dimension of the context space (\Cref{t:upper-bound-full-main}).
	\item We prove the optimality of our result in the full-feedback setting, showing that no other algorithms can achieve a regret of smaller order than \biave{} (\Cref{t:lower-bound-full-main}).
	\item For the limited-feedback setting, we design the \exbis{} algorithm (\Cref{a:lipshitz:twobit:feedback}) and show an upper bound on its regret after $T$ interactions of order $T^{\frac{d+2}{d+4}}$ (\Cref{t:upper-bound-limited-main}).
	\item We prove the optimality of our result in the limited-feedback setting, showing that no other algorithms can achieve a regret of smaller order than \exbis{} (\Cref{t:lower-bound-limited-main}).
	\item Finally, we discuss an even stronger benchmark, known as \emph{first-best}, where the oracle is omniscient and chooses the optimal price at each interaction with exact knowledge of what \emph{the realizations of traders' valuations} will be, therefore, never missing a trading opportunity.
	We prove that the classic benchmark oracle that chooses prices with exact knowledge of the traders' distributions earns a gain from trade that is at least $1/2$ of that of this omniscient oracle (\Cref{t:approximation}). 
	This yields, in particular, that the performance of our learning algorithms is at most $1/2$ away from that of an omniscient oracle that knows everything about all traders.
	Furthermore, we prove that this $1/2$-approximation factor is unimprovable (\Cref{t:approximation-tight}).
	To the best of our knowledge, this is the first work on online learning in the brokerage setting where this kind of approximation result is achieved.
\end{enumerate}

\subsection{Related work}
The literature on bilateral trade is extremely rich and has experienced a steady growth since the fundamental work of \cite{myerson1983efficient}.
Classically, bilateral trade has been explored in the one-shot setting, mainly from a game-theoretic and approximation perspective \citep{Colini-Baldeschi16,Colini-Baldeschi17,BlumrosenM16,brustle2017approximating,colini2020approximately,babaioff2020bulow,dutting2021efficient,DengMSW21,kang2022fixed,archbold2023non}.
For a fairly complete overview on this literature, see, e.g., \cite{cesa2023bilateral}.
On the other hand, a recent stream of literature explored bilateral trade in a repeated setting through the lens of online learning.
Being the most relevant for our work, we focus on this literature.

In \cite{cesa2021regret,azar2022alpha,cesa2023bilateral,cesa2023repeated,bernasconi2023no,cesa2024regret}, the authors examined the non-contextual repeated bilateral trade problem with predefined seller and buyer roles: at each interaction, a new seller/buyer pair arrives, the broker proposes a trading price, and the current item is traded if and only if the proposed price is higher than the private valuation of the seller and lower than the private valuation of the buyer.
When this happens, the buyer pays the trading price to the seller, the seller gives the item to the buyer, and the broker is rewarded with the gain from trade, i.e., the sum of the seller's and buyer's utility.
In \cite{cesa2021regret, cesa2023bilateral}, the authors investigated and obtained sharp regret bounds when sellers' and buyers' valuations, represented by two random sequences of numbers $(S_t)_{t \in \N},(B_t)_{t \in \N}$, form two i.i.d.\ sequences, while showing that the adversarial case is unlearnable in general.
\cite{azar2022alpha} managed to obtain learnability in the adversarial case by relaxing the notion of regret to the one of $2$-regret.
When the platform can post two different prices to sellers and buyers, but still not being allowed to subside trades, \cite{cesa2023repeated,cesa2024regret} achieved learnability using the usual notion of regret when the adversary belongs to the class of \emph{smoothed} adversaries.
\cite{bernasconi2023no} managed to achieve learnability in the adversarial case by allowing the platform to subsidize trade, as long as the subsidizing comes from revenue obtained from previous sellers and buyers interactions.
On a different direction, \cite{bachoc2024fair} investigated how to achieve \emph{fairness} in the repeated bilateral trade problem between seller's and buyer's earnings by rewarding the broker with the minimum between the seller's utility and the buyer's utility, instead of the gain from trade.

\cite{bolic2023online} introduced the non-contextual repeated bilateral trade setting where the traders have unspecified seller's and buyer's role (brokerage) and obtained sharp learning rates in the i.i.d.\ setting when the reward function is the gain from trade.
\cite{cesari2024trading} focused on the same non-contextual setting, but with the different objective of maximizing the total number of trades.
By investigating a contextual bilateral trade with unspecified seller and buyer roles with the gain from trade as reward function, \cite{bachoc2024contextual} is the closest to our setting.
However, they assume a linear (and hence, parametric) relationship between contexts and marker values, while we relax this assumption by imposing only that close enough contexts give rise to close enough market values.

Our work also shares some similarities with contextual bandits \citep{slivkins2011contextual}. 
Our largest differentiation compared to contextual bandits is provided by our results in the limited feedback setting (\Cref{s:limited-feedback}, which we recommend reaching to better understand the following discussion). 
There, the concepts of exploration and exploitation are different from the contextual bandit setting. With bandits, exploitation consists in choosing an arm which is currently expected to be the best. 
Nevertheless, the feedback of exploitation rounds is still a realization from an arm and can be used to learn the arm mean. 
In contrast, in our setting, exploitation consists in playing our estimate of the mean $\mu_t$ for context $\boldsymbol{x}_t$, and this has no value at all for learning the mean. 
Hence, our feedback for exploitation rounds are not used later on in \Cref{a:lipshitz:twobit:feedback} (ExBis), since the mean estimates on Lines 15 and 16 only use feedback from exploration rounds. 
In contrast, in \citep{slivkins2011contextual}, Algorithm 1, Line 12, each feedback can be used to update the various estimators for the algorithm. Regarding exploration, our exploration rounds consist in playing uniformly distributed prices (which is a feature specific to our limited feedback in the brokerage problem, and absent from bandit algorithms). 
Only the feedback from these rounds can be used, but we bound their instantaneous regret by the maximal possible value 1. In contrast, with bandits algorithms, exploration would consist in playing an arm which is not expected to be optimal, but can still be close to optimal with a regret bound much better than 1. 
On exploration/exploitation, another difference is that in our setting each round is tagged as one or the other (Lines 14 and 17 in Algorithm 3), while in \citep{slivkins2011contextual} and more broadly with bandits, each round is a mix of both, for instance with an upper confidence bound rule (see \citealt[Section 4.2]{slivkins2011contextual}). 
Finally, the multiple differences above have an impact on the regret rates that are achieved. 
In \citep{slivkins2011contextual}, the worst rate is $T^{(1+d_c)/(2+d_c)}$ (see (7), there) for the dimension $d_c$ of the space where the Lipschitz mean function is defined. 
In our setting, the rate is $T^{(d+2)/(d+4)}$.

\subsection{Techniques and challenges}
In the full-feedback case, \Cref{l:structural} in \Cref{s:approximation-appe} can be used to reduce our problem to a full-feedback online adversarial contextual regression problem where, at each time step $t$, the learner is presented with a new context $\bx_t$, asked to make a prediction $y_t$, suffers a corresponding loss $\ell_t(y_t)$, then observes $\ell_t$. 
Existing techniques \citep{hazan2007online,cesa2017algorithmic} study this problem with the goal of competing against the best Lipschitz policy that maps contexts to predictions.
However, a black-box application of these techniques requires noiseless feedback, i.e., the learner needs to reconstruct exactly the loss function at the end of each round (which in our setting would be the \emph{expected} gain from trade function).
In contrast, in our setting, after having access to the traders' valuations $V_t$ and $W_t$, we observe only noisy realizations of the market price $\mu_t$, which in turn translates into the fact that we observe only \emph{noisy} realizations of the associated  loss functions in the aforementioned reduction.
To circumvent this problem, we take a different route and devise \emph{ad hoc} techniques to estimate the value of our reward function at specific points.
Specifically, we partition the context space in dyadic cells, and use the feedback we receive from previous rounds to predict the value of the reward function for contexts that belong to the same cell.
Importantly, when sufficiently many points land in the same region, the dyadic partition is adaptively refined to increase the precision of the estimates by relying only on the information retrieved from the closest contexts.
A suitable choice of the criterion to further split the dyadic cells gives the optimal rate.
For the lower bound, it is important to note that we do not have direct control on the reward functions, but we need to devise suitable instances of traders' distributions in order to produce hard instances for our problem.
Once this is done, a lattice of sufficiently-spaced contexts is built and the contexts on this lattice are repeatedly presented to the learner.
Finally, we determine a suitable horizon-dependent tuning of the number of points in this lattice and of the number of times that each of these contexts should be repeatedly presented so that the learner cannot infer any non-trivial information about the market value associated to the next context in the lattice.

In the limited feedback case there is a further layer of complexity: 
the feedback we receive
depends on the posted price
and is not even enough to directly reconstruct \emph{bandit} feedback, i.e., the realized gain associated with the action we performed.
For this reason, we cannot directly rely on existing techniques to solve bandit online adversarial contextual regression problems, but we need to devise novel techniques tailored to our problem.
Specifically, we first show that a Monte Carlo sampling procedure can be performed to reconstruct an estimate of the market value associated to a certain context.
On the other hand, it is important to note that this exploration procedure is costly (it requires posting prices with low gain from trade).
By relying on an adaptive dyadic partition of the context space, we show how to properly balance exploration rounds (where we use this Monte Carlo procedure to estimate the reward function) and exploitation rounds (where we use this information to increase our total reward) to obtain optimal regret bounds.
The lower bound construction relies on the same lattice construction we used in the full-feedback case, but with the further layer of complexity of devising instances where, given the limited feedback, potentially optimal actions do not reveal \emph{any} meaningful information about their actual optimality, forcing the learner to explore in costly regions in order to obtain that piece of information, which loosely resembles the exploration/exploitation dilemma of the so-called revealing action problem \citep{cesa2006prediction}. 
This further layer of complexity requires a different tuning (with respect to the full feedback case) of the points in the lattice and the number of rounds in a row that contexts are presented.

Finally, the $1/2$-approximation result of the \emph{first-best} cannot be deduced from the corresponding existing literature on bilateral trade: the closest result to ours in this literature provides approximations of the \emph{first-best} when the traders have definite seller and buyer roles, and, crucially, it is required that they share the same distribution \citep{kang2019fixed}.
In contrast, in our case, each trader is allowed to sell and buy, and while they share the same expected value, they do not share the same distribution.
For this reason, devise an entirely new proof to deduce our approximation result of the \emph{first-best}.

\section{FULL FEEDBACK}
\label{s:full}
In this section, we analyze how efficiently the broker can learn by leveraging full feedback.

\subsection{\biave{} and regret upper bound}
\label{s:biave-upper-bound}

In this section, we introduce and analyze our \biave{} algorithm for the full-feedback setting.
We begin with some notation.
For a subset $\cS \subseteq \mathbb{R}$, we let $\cS^- = \inf \cS$, $\cS^+ = \sup \cS$, and $\length(\cS) = \cS^+ - \cS^-$.
For a subset $\cS\s\bbR^d$ of the form $\cS \ceq \cI_1 \times \cdots \times \cI_d$, where $\cI_1, \ldots , \cI_d$  are left-closed and right-open intervals of same length, we define $\length(\cS)$ as this common length, and we define $\bisect(\cS)$ as the set containing the $2^d$ hypercubes specified below
\[
\bisect(\cS) \ceq
\bcb{ \cI_{1,a_1} \times \dots \times \cI_{d,a_d} \mid
	a_1 , \ldots , a_d \in \{-,+\} },
\]
where, for $j\in [d]$, $\cI_{j,-} \ceq \bigl[\cI_j^- , (\cI_j^- + \cI_j^+)/2 \bigr)$ and $\cI_{j,+} \ceq \bigl[ (\cI_j^- + \cI_j^+)/2 , \cI_j^+ \bigr)$.
We consider dyadic hypercubes---which for brevity, we call \emph{cells} in what follows--- of the form $\prod_{j=1}^d \bigl[k_j 2^{-i},(k_{j}+1) 2^{-i} \bigr)$, for $k_1,\dots,k_d \in \{0,\dots,2^{i}-1\}$ and $i \in\{ 0,1,2,\dots \}$, and we say that this $i$ is the \emph{level} of the cell.
Given a family of cells $\sF$, we say that $\cC \in \sF$ is \emph{terminal} if no other cell in the family is properly contained in $\cC$.
If $\cC$ is a cell of level $i \ge 1$, its \emph{parent} is the only cell of level $i-1$ that contains it, which we denote by $\cC'$.
By convention, we say that the cell $[0,1)^d$ is the parent of itself, so if $\cC = [0,1)^d$ then $\cC' = \cC$.
The pseudoocde of our \biave{} algorithm is present in \Cref{a:lipshitz:full:feedback}.
\begin{algorithm}
	\caption{\biave{} (Bisect and average)}
	\textbf{initialization:} 
	Set $\sF \coloneqq \bcb{ [0,1)^d }$
	\begin{algorithmic}[1]
		\State The environment reveals a context $\bx_1 \in [0,1)^d$
		\State Post price $P_1:= \frac{1}{2}$
		\State Observe $V_1$ and $W_1$
		\For{$t=2,3,\dots$}
		\State The environment reveals a context $\bx_t \in [0,1)^d$ \label{line:envir:xt}
		\State Let $\cC_t$ be the terminal cell in $\mathscr{F}$ such that $\bx_t \in \cC_t$\label{state:set:Ct}
		\State Let $i_t$ be the level of $\cC_t$\label{state:it}
		\State Let $n_t$ be the number of $s\in[t-1]$ such that $\bx_s \in \cC_t$\label{state:nt}
		\If{$\cC_t = [0,1)^d$}
		\State Let $n'_t \ceq n_t$ 
		\Else
		\State Let $q_t$ be the time at which the parent cell $\cC'_t$ 
		
		of $\cC_t$ was bisected
		\State Let $n'_t$ be the number of $s\in[q_t]$ s.t.\ $\bx_s \in \cC'_t$
		\EndIf
		\If{ $n_t \ge n'_t$}\label{state:if:more:in:cell}
		\State 
		\label{state:if:more:in:cell:Pt}Post price $P_t \ceq \frac{1}{2n_t} \sum_{s=1}^{t-1} \I\{\bx_s \in \cC_t \} (V_s+W_s) $ \label{state:play-Pt-1}
		\ElsIf{$n_t < n'_t$}\label{state:if:more:in:parent}
		\State
		\label{state:if:more:in:parent:Pt}
		Post price $P_t \ceq \frac{1}{2n'_t} \sum_{s=1}^{q_t} \I\{\bx_s \in \cC'_t \} (V_s+W_s) $\label{state:play-Pt-2}
		\EndIf
		\If{ $\sqrt{n_t} \ge 2^{i_t}$  }\label{state:bisection:condition}
		\State Add the family of cells $\bisect( \cC_t )$ to $\sF$ \label{state:biAve-bisect}
		\EndIf
		\State Observe $V_t$ and $W_t$
		\EndFor \label{s:repeat-end}
	\end{algorithmic} \label{a:lipshitz:full:feedback}
\end{algorithm}
At a high level, the algorithm discretizes the context space $[0,1)^d$ into cells of \emph{locally-adaptive granularity}.
The price $P_t$ proposed at any time step $t$ is computed as one of two possible empirical averages, depending on whether or not the number of past contexts that fell into the terminal cell $\cC_t$ containing the current context $\bx_t$ is larger than the number of contexts that fell into the parent cell $\cC'_t$ of $\cC_t$ before $\cC'_t$ was bisected (\Cref{state:if:more:in:cell,state:if:more:in:parent}).
If sufficiently many contexts fell into $\cC_t$, then $P_t$ is chosen as the empirical average of all valuations observed in past rounds $s$ where contexts $\bx_s$ fell into $\cC_t$ (\Cref{state:play-Pt-1}).
Otherwise, $P_t$ is chosen as the empirical average of the valuations observed in past rounds $s$ where contexts $\bx_s$ fell into the parent cell $\cC'_t$, up to the time where $\cC'_t$ was bisected (\Cref{state:play-Pt-2}).
Finally, as soon as ``too many'' contexts have fallen into the same terminal cell (\Cref{state:bisection:condition}), the algorithm bisects it to increase the granularity of the estimation in that context region (\Cref{state:biAve-bisect}).
We now provide theoretical guarantees for the performance of \biave{}.

\begin{theorem}
	\label{t:upper-bound-full-main}
	In the full-feedback setting, if we run the \biave{} algorithm for $T$ time steps, its regret satisfies
	\[
	R_T
	=
	O \lrb{T^{\frac{d}{d+2}}}.
	\]
\end{theorem}

For a full proof of this result, see \Cref{section:simplified:fullfeed}.

\begin{proof}[Proof sketch]
	We begin by claiming that the optimal price to propose at any time $t\in \N$ is $P_t \ceq \mu_t$, where $\mu_t$ is the market price at time $t$ (see \Cref{l:structural} in \Cref{s:approximation-appe}), and that posting any other price $p$ would result in a instantaneous regret of order $O\brb{ (\mu_t - p)^2 }$, i.e., quadratic in the distance from the market price $\mu_t$ (again, see \Cref{l:structural} in \Cref{s:approximation-appe}).
	We also note that the updates of $\sF$ during a run of the algorithm are deterministic, since they only depend on the adversarial sequence of contexts $\bx_1,\bx_2,\dots$.
	Using these facts, we can prove that the two different rules the algorithm uses to determine its proposed prices $P_t$ (on \Cref{state:play-Pt-1,state:play-Pt-2}) are sufficiently accurate approximations of $\mu_t$.

	Given that both rules are empirical averages of past traders' valuations coming from rounds in which contexts fell in a common cell, this boils down to quantifying the bias and variance. 
	First, although the empirical averages are \emph{biased} estimates of market prices $\mu_t$, this bias can be controlled. In the case of \Cref{state:if:more:in:cell:Pt}, all the observations $V_s$, $W_s$ are associated to contexts $\bx_s$ in a cell of diameter $O\brb{2^{-i_t}}$, so that by the  Lipschitz property of $\mu_t$, the squared bias has order $O\brb{2^{-2 i_t}}$. 
	By independence, the variance has order $O\brb{\frac{1}{n_t}}$.
	Our bisection condition at \Cref{state:bisection:condition} then ensures that the variance also has order $O\brb{2^{-2 i_t}}$.
	We have a similar conclusion for \Cref{state:if:more:in:parent:Pt}.
	
	The last main step of the proof is to consider any potential cell $\cC$ and to count the number  $n_\cC$  of time steps $t\in[T]$ where this cell is equal to $\cC_t$ on \Cref{state:set:Ct}. Then we can show $n_\cC \le 2^{2 i_\cC}$ with $i_\cC$ being the level of $\cC$ (from \Cref{state:bisection:condition}). By partitioning the cumulated regret according to all these possible cells $\cC$, that can be indexed as $\{\cC_{i,j}\}_{i=0,j=1}^{\infty,2^{id}}$, we obtain an upper bound of the order
	\[
	\sum_{i=0}^{\infty}
	\sum_{j=1}^{2^{id}}
	n_{\cC_{i,j}}
	2^{-2i},
	\]
	with the constraints that
	\[
	n_{\cC_{i,j}} \le 2^{2 i}
	\quad \text{and} \quad 
	\sum_{i=0}^{\infty}
	\sum_{j=1}^{2^{id}} n_{\cC_{i,j}} \le T.
	\]
	Some final technicalities yield the theorem.
\end{proof}

\subsection{Regret lower bound: \biave{} is optimal}

In this section, we prove the optimality of \biave{} in the full-feedback setting, by showing that any other algorithm will pay a 
regret of order at least $T^{\frac{d}{d+2}}$.

\begin{theorem}
	\label{t:lower-bound-full-main}
	In the full-feedback setting, for any time horizon $T$, any algorithm suffers 
	regret
	\[
	R_T
	=
	\Omega \lrb{ T^{\frac{d}{d+2}} }.
	\]
\end{theorem}

For a full proof of this result, see \Cref{s:proof-lower-bound-full}.

\begin{proof}[Proof sketch]
	The key idea of the lower bound is to pack the context space $[0,1)^d$ with a lattice of $k$ equispaced points.
	The environment then selects contexts as follows.
	For any given time horizon $T$, it begins by revealing one of the contexts $\bx'$ for $T/k$ consecutive rounds, then moves on to a different context $\bx''$ and reveals this second context for $T/k$ rounds, and so on until all $k$ contexts in the lattice have been revealed for $T/k$ rounds each.
	This way, for each of the $k$ points in the lattice, a learner will observe $T/k$ noisy realizations of its corresponding market value.
	Leveraging this construction, we show that no learner is able to confidently distinguish market values corresponding to consecutive points in the lattice (and, \emph{a fortiori} this cannot be done for contexts that are further away) if their corresponding market values are too close.
	The idea is then to select a sequence of market values compatible with the contexts (i.e., that are close enough if contexts are close enough) that are still far enough to guarantee a sufficiently high regret.
	We do this by setting a threshold level and, for each context in the lattice, randomly raising or lowering the corresponding market value by a small constant.
	The result then follows by tuning the number of elements in the lattice and by proving that these random perturbations of a threshold market value can be done in a way that respects our modeling assumptions.
\end{proof}

\section{LIMITED FEEDBACK}
\label{s:limited-feedback}

In this section, we analyze how efficiently the broker can learn by leveraging limited feedback.

\subsection{\exbis{} and regret upper bound}
\label{s:exbis-and-upper-bound}

In this section, we introduce and analyze our \exbis{} algorithm for the limited-feedback setting.
In addition to the bisection notation presented in \Cref{s:biave-upper-bound}, we introduce here the i.i.d.\ sequence of $[0,1]$-valued uniform random variables $U_1,U_2,\dots$ that are used by \exbis{} as random seeds, and that are independent of the sequence of traders' valuations $V_1,W_1, V_2, W_2, \dots$.
The pseudocode of our \exbis{} algorithm is presented in \Cref{a:lipshitz:twobit:feedback}.
\begin{algorithm}
	\caption{\exbis{} (Exploit, Explore, and Bisect)}
	\textbf{initialization:} Set $\sF \coloneqq \bcb{ [0,1)^d }$
	\begin{algorithmic}[1]
		\State The environment reveals a context $\bx_1 \in [0,1)^d$
		\State Post $P_1 \ceq U_1$ and add $\bisect \brb{ [0,1)^d }$ to $\sF$
		\State Observe $\Vtilde_1 \ceq \I\{P_1\le V_1\}$ and $\Wtilde_1 \ceq \I\{P_1\le W_1\}$
		\For{$t=2,3,\dots$}
		\State The environment reveals a context $\bx_t \in [0,1)^d$ \label{line:envir:xt-exbis}
		\State Let $\cC_t$ be the terminal cell in $\mathscr{F}$ such that $\bx_t \in \cC_t$\label{state:set:Ct:defined}
		\State Let $i_t$ be the level of $\cC_t$\label{state:set:Ct:twobit}
		\State Let $m_t$ be the number of $s\in[t-1]$ such that 
		
		$m_s < 2^{4i_s}$  (exploiting) \textbf{and} $\cC_s = \cC_t$\label{state:Mt}
		\State Let $\cE_t$ be the set of $s \in [t-1]$ such that 
		
		$m_s \ge 2^{4i_s}$ (exploring) \textbf{and} 
		$\bx_s \in \cC_t$
		\State Let $n_t$ be the cardinality of $\cE_t$
		\State Let $q_t$ be the time at which the parent cell $\cC'_t$ of 
		
		$\cC_t$ was bisected
		\State Let $\cE'_t$ be the set of $s \in [q_t]$ such that 
		
		$m_s \ge 2^{4i_s}$ (exploring) \textbf{and} 
		$\bx_s \in \cC'_t$
		\State Let $n'_t$ be the cardinality of $\cE'_t$
		\If{ $m_t < 2^{4 i_t}$\label{state:if:exploit}}\Comment{(exploiting)}
		\State \textbf{if} $n_t \ge n'_t$ \textbf{then} 
		Post $P_t \ceq \frac{1}{2 n_t} \sum_{s \in \cE_t} (\Vtilde_s + \Wtilde_s)$ \label{state:if:first:exploit}\label{state:first:exploit}
		\State \textbf{if} $n_t < n'_t$ \textbf{then} 
		Post $P_t \ceq \frac{1}{2n'_t} \sum_{s \in \cE'_t} ( \Vtilde_s + \Wtilde_s )$\label{state:if:second:exploit}\label{state:second:exploit}
		\ElsIf{$m_t \ge 2^{4 i_t}$\label{if:explore:and:bisect}}\Comment{(exploring)}
		\State Post $P_t \ceq U_t$ \label{state:posting-exploration}
		\State \textbf{if} $n_t \ge 2^{2 i_t} -1$ \textbf{then}
		Add $\bisect( \cC_t )$ to $\sF$ \label{state:bisect}
		\EndIf
		\State Observe $\Vtilde_t \ceq \I\{P_t\le V_t\}$ and $\Wtilde_t \ceq \I\{P_t\le W_t\}$
		\EndFor	
	\end{algorithmic}
	\label{a:lipshitz:twobit:feedback}\vspace{-0.2ex}
\end{algorithm}
The algorithm leverages the adaptive granularity insights of \biave{}, with the additional challenges arising from the limited feedback.
Readers familiar with multiarmed bandits might have noticed that the limited feedback is less informative than the already limited bandit feedback, not even being sufficient to compute the reward function at the posted price.
To get around this roadblock, \exbis{} reserves exploration rounds (\Cref{if:explore:and:bisect}) where uniform prices are posted to allow  gathering estimates of the market value at the cost of a high instantaneous regret.
In the exploitation rounds (\Cref{state:if:exploit}), instead, the algorithm posts an empirical average of the estimates of the market value gathered in past rounds where contexts were close enough to the current one.
Finally, the algorithm locally increases the granularity of the estimates by bisecting areas where sufficiently many contexts have fallen.
We now provide theoretical guarantees for the performance of \exbis{}.

\begin{theorem}
	\label{t:upper-bound-limited-main}
	In the limited-feedback setting, if we run the \exbis{} algorithm for $T$ time steps, its regret satisfies
	\[
	R_T
	=
	O \lrb{ T^{\frac{d+2}{d+4}} }.
	\]
\end{theorem}

For a full proof of this result, see \Cref{section:simplified:twobitfeed:oned}.

\begin{proof}[Proof sketch]
	Similarly to the proof of \Cref{t:upper-bound-full-main}, we begin by observing that the optimal price to propose at any time $t$ is the market price $\mu_t$, that posting any other price $p$ would result in an instantaneous regret of order $O\brb{ (\mu_t - p)^2 }$, and that the updates of $\sF$, as well as all the quantities appearing in the algorithm (with the only exceptions of $\Vtilde$, $\Wtilde$, and $P_t$) are deterministic, since they only depend on the adversarial sequence of contexts $\bx_1,\bx_2,\dots$.
	
	Similarly as when proving Theorem \ref{t:upper-bound-full-main}, we consider any potential cell $\cC_{i,j}$ and the cumulated regret over all time steps $t \in [T]$ where this cell is equal to $\cC_t$ on \Cref{state:set:Ct:defined} of \Cref{a:lipshitz:twobit:feedback}. An important difference compared to \Cref{t:upper-bound-full-main} is that now the regret comes both from exploration and exploitation rounds. For the sequence of time steps $t \in [T]$ where $\cC_{i,j} = \cC_t$ on \Cref{state:set:Ct:defined},
	\Cref{a:lipshitz:twobit:feedback} starts with many exploitation rounds (\Cref{state:if:exploit}), and only when a sufficient number of them is achieved, do exploration rounds occur (\Cref{if:explore:and:bisect}). As a consequence, we can show that it is sufficient to bound the regret stemming from exploitation rounds only. 
	
	As for the proof of \Cref{t:upper-bound-full-main}, the instantaneous regret of one exploitation round is $O(2^{-2i})$.
	The number of these exploitation rounds for $\cC_{i,j}$, $n_{\cC_{i,j}}$, is bounded by $O(2^{4i})$
	from \Cref{if:explore:and:bisect} (in the proof of \Cref{t:upper-bound-full-main}, the bound was $O(2^{2i})$ which explains the final difference of order of bounds between the two theorems).
	In the end, we obtain
	an upper bound of the order
	\[
	\sum_{i=0}^{\infty}
	\sum_{j=1}^{2^{id}}
	n_{\cC_{i,j}}
	2^{-2i},
	\]
	with the constraints that
	\[
	n_{\cC_{i,j}} \le 2^{4 i}
	\quad \text{and} \quad 
	\sum_{i=0}^{\infty}
	\sum_{j=1}^{2^{id}} n_{\cC_{i,j}} \le T.
	\]
	As for the proof of \Cref{t:upper-bound-full-main}, some final technicalities yield the theorem.
\end{proof}

\subsection{Regret lower bound: \exbis{} is optimal}

In this section, we prove the optimality of \exbis{} in the full-feedback setting, by showing that any other algorithm will pay a 
regret of order at least $T^{\frac{d+2}{d+4}}$.

\begin{theorem}
	\label{t:lower-bound-limited-main}
	In the limited-feedback setting, for any time horizon $T$, any algorithm suffers 
	regret
	\[
	R_T
	=
	\Omega \lrb{ T^{\frac{d+2}{d+4}} }.
	\]
\end{theorem}

For a full proof of this result, see \Cref{s:proof-lower-bound-limited}

\begin{proof}[Proof sketch]
	For this lower bound, we leverage the same lattice construction we used in the full-feedback case.
	The main difference is that now the feedback depends on the algorithm in a way that allows us to build two different sequences of traders' valuations distributions (i.e., instances) at each lattice point with the following high-level properties: A price $p_1$ is optimal in the first instance and suboptimal on the second; A price $p_2$ is optimal in the second instance and suboptimal in the first; Neither $p_1$ nor $p_2$ reveal any meaningful feedback; there exists a third price $p_0$, which is highly suboptimal in both instances but, every time it is chosen, it reveals some information about the underlying instance being the first or the second one.
	Tuning everything properly and showing that one such construction can be obtained without violating our modeling assumptions gives the result.
\end{proof}

\section{ \texorpdfstring{$\frac{1}{2}$}{(1/2)}-APPROXIMATION OF FIRST-BEST}
\label{s:one-half-approx}

In this section, we show that our theory yields a $\frac12$-approximation of the performance of an omniscient oracle with perfect information about the \emph{realizations} of the traders' valuations. 
This powerful oracle is known in game theory and economics as \emph{first-best}. We also prove that our $\frac12$-approximation of the first-best is tight, i.e., that no approximation rate better than $\frac12$ can be obtained in general.

\begin{theorem}
	\label{t:approximation}
	Suppose that $V$ and $W$ are two bounded non-negative independent random variables admitting bounded densities, with cumulative distributions $F$ and $G$, respectively.
	Assume that $\E[V] = \E[W] \eqqcolon \mu$.
	Then
	\begin{multline*}
		\max_{p\in[0,1]} \E\bsb{\gft(p,V,W)}
		=
		\E\bsb{\gft(\mu,V,W)} 
		\ge 
		\frac{1}{2} \cdot \E\bsb{|W-V|}
		=
		\frac{1}{2} \cdot \E\lsb{ \max_{p\in[0,1]} \gft(p,V,W)}
		\;.
	\end{multline*}
\end{theorem}

\begin{proof}
	Integrating by part twice, we get
	{\allowdisplaybreaks
		\begin{align*}
			&
			\int_0^{+\infty} F(\lambda) \lrb{1-G(\lambda)}\dif \lambda
			\\
			&=
			\lim_{u \to \infty }\lsb{\int_0^{\lambda}F(v) \dif v \lrb{1-G(\lambda)}}_{\lambda = 0}^u 
			+ \int_0^{+\infty}  \int_0^{\lambda}F(v) \dif v \dif G(\lambda)
			\\
			&=
			\int_0^{+\infty} \int_0^{\lambda}F(v) \dif v \dif G(\lambda) 
			=
			\E\lsb{\int_0^W F(v) \dif v}
			\\
			&=
			\E \lsb{ \bsb{-(W-v)F(v)}_{v=0}^W + \int_0^W (W-v)\dif F(v) }
			\\
			&=
			\E \lsb{ \int_0^W \!\! (W-v)\dif F(v) }
			\\
			&
			=
			\E \lsb{ \int_0^{+\infty} \!\! (W-v) \I\{v\le W\}\dif F(v) }
			\\
			&=
			\E \Bsb{ \E \bsb{ (W-V) \I\{V\le W\} \mid W} }
			\\
			&
			=
			\E \bsb{ (W-V) \I\{V\le W\} } \;.
		\end{align*}
	}%
	Analogously, switching the role of $V$ and $W$, we can prove that
	\[
	\int_0^{+\infty} G(\lambda) \brb{1-F(\lambda)}\dif \lambda
	=
	\E \bsb{ (V-W) \I\{W\le V\} }\;. 
	\]
	It follows that
	\begin{align}
		\label{eq:best-price}
		&
		\E\bsb{|W-V|}
		=
		\E \bsb{ (W-V) \I\{V\le W\} } \nonumber
		+ 
		\E \bsb{ (V-W) \I\{W\le V\} } \nonumber
		\\
		&=
		\int_0^{+\infty} \!\!\!\!\!\! F(\lambda) \brb{1-G(\lambda)}\dif \lambda
		+
		\int_0^{+\infty} \!\!\!\!\!\! G(\lambda) \brb{1-F(\lambda)}\dif \lambda . 
	\end{align}
	Now, given that $(F+G)(0) = 0$, that $\lim_{u \to +\infty}(F+G)(u)=2$ and that $F+G$ is continuous, there exists (and we fix) $p \in (0,+\infty)$ such that $(F+G)(p)=1$. Then
	{\allowdisplaybreaks
		\begin{align*}
			&
			\E\bsb{\gft(\mu,V,W)}
			\ge
			\E\bsb{\gft(p,V,W)}
			\\
			&
			=
			\int_0^p (F+G)(\lambda)\dif \lambda + (\mu-p)(F+G)(p)
			\\
			&
			=
			\int_0^p (F+G)(\lambda)\dif \lambda + \mu-p
			\\
			&
			=
			\int_0^p \!\! (F+G)(\lambda)\dif \lambda + \frac{1}{2} \int_0^{+\infty} \!\!\!\!\!\!\! \brb{ 1-F(\lambda) + 1 - G(\lambda) } \dif \lambda-p
			\\
			&=
			\int_0^p (F+G)(\lambda)\dif \lambda + \frac{1}{2} \int_0^{p} \brb{ 1-F(\lambda) + 1 - G(\lambda) } \dif \lambda 
			\\
			&
			\quad 
			+ \frac{1}{2} \int_p^{+\infty} \brb{ 1-F(\lambda) + 1 - G(\lambda) } \dif \lambda -p
			\\
			&=
			\frac{1}{2}\int_0^p F(\lambda)\dif \lambda + \frac{1}{2}\int_0^p G(\lambda)\dif \lambda 
			\\
			&
			\quad
			+ \frac{1}{2} \int_p^{+\infty} \brb{ 1-F(\lambda)} \dif \lambda + \frac{1}{2} \int_p^{+\infty} \brb{1 - G(\lambda) } \dif \lambda
			\\
			&\ge
			\frac{1}{2}\int_0^p F(\lambda)\brb{1-G(\lambda)}\dif \lambda + \frac{1}{2}\int_0^p G(\lambda)\brb{1-F(\lambda)}\dif \lambda 
			\\
			&
			\quad
			+\frac{1}{2} \int_p^{+\infty} \!\!\!\!\!\! G(\lambda)\brb{ 1-F(\lambda)} \dif \lambda + \frac{1}{2} \int_p^{+\infty} \!\!\!\!\!\!F(\lambda)\brb{1 - G(\lambda) } \dif \lambda
			\\
			&
			=
			\frac{1}{2}\int_0^{+\infty} \!\!\!\!\!\! F(\lambda)\brb{1-G(\lambda)}\dif \lambda + \frac{1}{2}\int_0^{+\infty} \!\!\!\!\!\! G(\lambda)\brb{1-F(\lambda)}\dif \lambda
			\\
			&
			=
			\frac{1}{2} \cdot \E\bsb{|W-V|} \;,
		\end{align*}
	}%
	where the first inequality and the first equality follow from \cref{l:structural} in \Cref{s:approximation-appe}, the second equality by the definition of $p$, the second inequality from the fact that $0\le F \le 1$ and $0 \le G \le 1$, and the last equality from \cref{eq:best-price}.
\end{proof}

The following lemma shows that the previous $\frac12$-approximation guarantee is unimprovable in general.

\begin{theorem}
	\label{t:approximation-tight}
	For each $\varepsilon>0$, there exist two independent $[0,1]$-valued random variables $V$ and $W$ admitting bounded densities and with common expectation $\mu$ such that
	\[
	\E\bsb{\gft(\mu,V,W)}
	\le
	\lrb{\frac{1}{2}+\varepsilon} \cdot \E\bsb{|W-V|}\;.
	\]
\end{theorem}

For a full proof of this result, see  \Cref{s:approximation-appe}.

\begin{proof}[Proof sketch]
	The proof leverages the fact that $V$ and $W$ are free to have different distributions as long as they share the same expected value.
	The idea is to determine two different distributions (say, the first for $V$ and the second for $W$) whose shared expectation is $1/2$ and, while the first is symmetric around $1/2$ and highly concentrated near $0$ and $1$, the second one is still symmetric around $1/2$ but highly concentrated around $1/2$.
	From \cref{l:structural} (in \Cref{s:approximation-appe}) we know that the best fixed price is $1/2$.
	By posting $1/2$, the broker manages to let trades happen when $V$ has a value close to $0$ and $W$ has a value slightly greater than $1/2$, or when $V$ has a value close to $1$ and $W$ has a value slightly smaller than $1/2$, leading to an expected reward that is slightly above $1/4$.
	On the other hand, the first-best manages to let trades happen in all the previous cases, but also when $V$ is close to $0$ and $W$ is slightly smaller than $1/2$ and when $V$ is close to $1$ and $W$ is slightly bigger than $1/2$, for an extra expected reward that is slightly below $1/4$.
\end{proof}

\section{CONCLUSIONS AND FUTURE WORK}

In this paper, we investigated a \emph{Lipschitz} contextual brokerage problem, extending the classical brokerage problem to a non-parametric contextual setting. 
We designed two algorithms, \biave{} and \exbis{}, to minimize regret in full and limited feedback settings, respectively. Our results provide tight regret bounds, specifically $ O(T^{d/(d+2)}) $ for the full-feedback setting and $ O(T^{(d+2)/(d+4)}) $ for the limited-feedback setting, demonstrating the optimality of these approaches. 
Furthermore, we established a $\frac12$-approximation factor between the performance of our algorithms and an omniscient oracle, proving that this approximation is unimprovable.

Our findings offer significant theoretical contributions to the study of brokerage problems in online learning and commerce applications involving contextual information. 
By relaxing parametric assumptions and focusing on non-parametric methods, we offer broad applicability to real-world over-the-counter (OTC) markets, where trade conditions and valuations are often influenced by contextual factors.

A natural extension of this work would involve relaxing the assumptions about market value fluctuations to accommodate more general stochastic processes. 
This could allow the model to be applied in more volatile or uncertain market environments.
Finally, while this work fully fleshed out the full and limited feedback settings, future work could explore other feedback models, such as partial trader disclosures, which could yield additional insights into regret minimization approaches in economics.




\appendix

\section{PROOF OF THEOREM \ref{t:upper-bound-full-main}}
\label{section:simplified:fullfeed}

Note first that the evolution of $\sF$ during a run of \biave{} is deterministic, since the decision to refine it or not (\Cref{state:bisection:condition}) depends only on the adversarial sequence of contexts $\bx_1,\bx_2,\dots$. 
For the same reason, for any time step $t\in\N$, $\cC_t$, $\cC_t'$, $i_t$, $n_t$, $q_t$ (when defined), and $n'_t$ are deterministic. 
Then, for any time step $t$ where the property $n_t \ge n'_t$ on \Cref{state:if:more:in:cell} holds (which is, again, a deterministic event), \Cref{l:structural} implies that the instantaneous regret of \biave{} satisfies
\begin{align*}
	&
	\sup_{p_t\in[0,1]} \E \bsb{ \gft ( p_t, V_t, W_t ) } - \E \bsb{ \gft ( P_t, V_t, W_t ) }
	\\
	&
	\qquad
	\le
	M \E \left[ 
	\left(
	\mu_t
	-
	\frac{1}{2n_t} \sum_{s=1}^{t-1} \I\{\bx_s \in \cC_t \} (V_s + W_s)
	\right)^2
	\right] 
	\\
	&
	\qquad
	= 
	M \lrb{ \E \left[ 
		\mu_t
		-
		\frac{1}{n_t} \sum_{s=1}^{t-1} \I\{\bx_s \in \cC_t \} \mu_s
		\right] }^2
	+
	M \Var \left[ \frac{1}{2n_t} \sum_{s=1}^{t-1} \I\{ \bx_s \in \cC_t \} (V_s + W_s) \right]
	\\ 
	& 
	\qquad
	\le
	M 
	2^{-2i_t}
	+ 
	\frac{M}{2n_t},
\end{align*}
where in the last inequality, we used the fact that the sum is over rounds $s\in[t]$ such that $\bx_s\in\cC_t$ and that cell $\cC_t$ has level $i_t$, which implies, by \Cref{i:two} in our model, that $\labs{\mu_s - \mu_t} \le 2^{-i_t}$. 
Moreover, in the same time steps $t$ where the property $n_t \ge n'_t$ on \Cref{state:if:more:in:cell} holds,
since $\sqrt{n'_t} \ge 2^{i_t-1}$ by \Cref{state:bisection:condition}, we obtain that $\sqrt{n_t} \ge 2^{i_t-1}$, which plugged into the right-hand side of the previous chain of inequalities yields that the instantaneous regret of \biave{} is upper bounded by $3M \cdot 2^{-2 i_t}$. 

We now prove that a similar bound holds in the complementary set of rounds $t\in\N$ in which property $n_t<n'_t$ on \Cref{state:if:more:in:parent} is true; indeed, in any of these rounds $t$, reasoning as above, the instantaneous regret of \biave{} satisfies
\begin{align*}
	&
	\sup_{p_t\in[0,1]} \E \bsb{ \gft ( p_t, V_t, W_t ) } - \E \bsb{ \gft ( P_t, V_t, W_t ) }
	\\
	&
	\qquad 
	\le 
	M \E \left[ 
	\left( 
	\mu_t
	-
	\frac{1}{2n'_t} \sum_{s=1}^{q_t} \I\{\bx_s \in \cC'_t \} (V_s + W_s)
	\right)^2
	\right] 
	\\
	&
	\qquad
	= 
	M \E \left[
	\mu_t
	-
	\frac{1}{n'_t} \sum_{s=1}^{q_t} \I\{\bx_s \in \cC'_t \} \mu_s
	\right]^2
	+
	M \Var \left[ \frac{1}{2 n'_t} \sum_{s=1}^{q_t} \I\{\bx_s \in \cC'_t \} (V_s + W_s) \right]
	\\ 
	& 
	\qquad
	\le 
	M 2^{-2(i_t-1)}
	+ 
	\frac{M}{2n'_t}.
\end{align*}
Because in time steps $t$ in which property $n_t<n'_t$ on \Cref{state:if:more:in:parent} is true, the parent cell $\cC'_t$ was bisected, we have from \Cref{state:bisection:condition} that $\sqrt{n'_t} \ge 2^{i_t - 1}$, which plugged into the right-hand side of the previous chain of inequalities yields that the instantaneous regret of \biave{} is upper bounded by $6M \cdot 2^{-2 i_t} $. 

Therefore, for \emph{all} time steps $t\in\N$, the instantaneous regret of \biave{} satisfies
\[
\sup_{p_t\in[0,1]} \E \bsb{ \gft ( p_t, V_t, W_t ) } - \E \bsb{ \gft ( P_t, V_t, W_t ) }
\le  
6M \cdot 2^{-2 i_t}.
\]
We now show that this is sufficient to prove that the regret (i.e., the worst-case sum over $t\in[T]$ of all instantaneous regrets) is upper bounded by $T^{d/(d+2)}$, up to constants. 

To see it, begin by considering any cell $\cC$ that can be obtained by successive bisections of $[0,1)^d$ (i.e., one of the cells that could be generated by \biave{}, introduced at the beginning of \Cref{s:biave-upper-bound}), and denote by $n_\cC$ the number of time steps $t\in[T]$ where this cell is equal to $\cC_t$ on \Cref{state:set:Ct}; 
then, from \Cref{state:bisection:condition}, we have $\sqrt{n_\cC} \le 2^{i_\cC}$, with $i_\cC$ being the level of $\cC$. 
Therefore, we can bound the regret as follows. 
For any level $i \in \N$, let $\cC_{i,1} , \ldots , \cC_{i,2^{id}}$ be the $2^{id}$ cells of the form $[b_1 , b_1 + 1/2^i) \times \dots \times [b_d , b_d + 1/2^i)$, for $b_1,\ldots,b_d \in \{0 , \dots, \frac{2^i - 1}{2^i}\}$. 
Then, putting everything together, we get
\begin{equation}
	\label{e:regret-ub-full-info-proof}
	R_T
	\le 
	6M
	\sum_{i=0}^{\infty}
	\sum_{j=1}^{2^{id}}
	n_{\cC_{i,j}}
	2^{-2i},
\end{equation}
with the constraints that (letting $\N_0 \ceq \{0,1,2,\dots\}$):
\begin{equation}
	\label{e:constraints}
	n_{\cC_{i,j}} \le 2^{2 i},
	\ \forall i \in \N_0, \forall j \in [2^{id}],
	\quad \text{and} \quad 
	\sum_{i=0}^{\infty}
	\sum_{j=1}^{2^{id}} n_{\cC_{i,j}} = T.
\end{equation}
Let $k$ be the smallest integer such that 
\[
\sum_{i=1}^k 
2^{id}
2^{2i} \ge T.
\]
Note that we have 
\[
\frac{1 - 2^{(k+1)(d+2)}}{1 - 2^{d+2}}
\ge T
\]
and thus
\[
2^{(k+1)(d+2)}
\ge 
1 + (2^{d+2}-1)T.
\]
Also we have, by definition of $k$,
\[
\sum_{i=1}^{k-1} 
2^{id}
2^{2i} < T.
\]
Thus
\[
\frac{1 - 2^{k(d+2)}}{1 - 2^{d+2}}
< T
\]
and consequently
\[
2^{k(d+2)}
< 
1 + (2^{d+2}-1)T.
\]
By the constraints in \Cref{e:constraints} and the definition of $k$, \Cref{e:regret-ub-full-info-proof} implies that
\begin{align*}
	\frac{1}{6M}
	R_T
	\le &
	\sum_{i=0}^k
	2^{id} 2^{2i} 2^{-2i}
	\\
	= &
	\frac{1 - 2^{(k+1)d}}{1-2^d}
	\\ 
	\le &
	\frac{1}{2^d-1}
	2^{(k+1)d}
	\\ 
	= &
	\frac{1}{2^d-1}
	2^{k(d+2)}
	2^{d-2k} 
	\\ 
	\le &
	\frac{1+(2^{d+1}-1)T}{2^d-1} 
	2^{d-2k}.
\end{align*}
Now, recalling that
\[
2^{(k+1)(d+2)}
\ge 
1 + (2^{d+2}-1)T,
\]
we obtain
\[
2^{k(d+2)}
\ge 
\frac{ 2^{d+2}-1
}{
	2^{d+2}
}
T
\]
and consequently
\[
2^{k}
\ge 
\left(
\frac{ 2^{d+2}-1
}{
	2^{d+2}
}
\right)^{\frac{1}{d+2}}
T^{\frac{1}{d+2}}.
\]
Hence
\begin{align*} 
	\frac{1}{6M}
	R_T
	&
	\le
	\frac{2^d\brb{ 1+(2^{d+1}-1) }}{2^d-1}
	T 
	2^{-2k}
	\\
	&
	\le 
	\frac{1+(2^{d+1}-1) }{
		\left(
		\frac{ 2^{d+2}-1
		}{
			2^{d+2}
		}
		\right)^{\frac{2}{d+2}}
	}
	\cdot
	T^{1-\frac{2}{d+2}}
	\\ 
	& =
	\frac{2^{d+1}}{
		\left(
		\frac{ 2^{d+2}-1
		}{
			2^{d+2}
		}
		\right)^{\frac{2}{d+2}}
	}
	\cdot
	T^{\frac{d}{d+2}}
	\\
	&
	=
	\frac{4\cdot2^{d+1}}{\left(2^{d+2}-1\right)^{\frac{2}{d+2}}}\cdot T^{\frac{d}{d+2}}
	\\
	&
	=    \frac{4\cdot2^{d+1}}{4\left(1-\frac{1}{2^{d+2}}\right)^{\frac{2}{d+2}}}\cdot T^{\frac{d}{d+2}}
	\\
	&
	\le
	\frac{2}{\left(\frac{7}{8}\right)^{2/3}}2^{d}\cdot T^{\frac{d}{d+2}},
\end{align*}
which, using the numerical inequality $6\cdot 2 \cdot \left(\frac{8}{7}\right)^{2/3} \le 14$, yields
\[
R_T
\le
14 \cdot M \cdot 2^{d} \cdot T^{\frac{d}{d+2}}
=
O \lrb{ T^{\frac{d}{d+2}} }
\]
and concludes the proof.

\section{PROOF OF THEOREM \ref{t:lower-bound-full-main}}
\label{s:proof-lower-bound-full}

Fix $T\in \N$.
Assume without loss of generality that $K \ceq T^{\frac{1}{d+2}}$ is an integer, and note that $K^d$ divides $T$.
Let $n \ceq \frac{T}{K^d} = T^{2/(d+2)} = K^2 \in \N$ and $\e \ceq n^{-1/2} = T^{-1/(d+2)}$.
Let
$
f_{\pm \e} 
\coloneqq
1
\mp \e \I_{\lsb{\frac17,\frac3{14}}} 
\pm \e \I_{ \left( \frac3{14},\frac27 \right] }
$.
Note that $0 \le f_{\pm\e} \le 2$ and $\int_0^1 f_{\pm\e}(x) \dif x = 1$, hence $f_{\pm\e}$ is a valid density on $[0,1]$ bounded by $M=2$.
We will denote the corresponding probability measure by $\cD_{\pm\e}$ and define $\mu_{\pm\e} \coloneqq \int_{[0,1]} x \dif \cD_{\pm\e}(x) = \frac12 \pm \frac{\e}{196}$.
Consider for each $q \in [0,1]$, an i.i.d.\ sequence $(B_{q,t})_{t \in \N}$ of Bernoulli random variables of parameter $q$, an i.i.d.\ sequence $(\tilde{B}_t)_{t \in \N}$ of Bernoulli random variables of parameter $1/7$, an i.i.d.\ sequence $(U_t)_{t \in \N}$ of uniform random variables on $[0,1]$, such that $\lrb{(B_{q,t})_{t \in \N, q \in [0,1]} , (\tilde{B}_t)_{t \in \N}, (U_t)_{t \in \N}}$ is an independent family. 
Let $\varphi \colon [0,1] \to [0,1]$ be such that, if $U$ is a uniform random variable on $[0,1]$, then the distribution of $\varphi(U)$ has density $\frac{7}{6} \cdot \I_{[0,1]\m [\fracc{1}{7},\fracc{2}{7}]}$ (which exists by the Skorokhod representation theorem, \citealt[Section 17.3]{williams1991probability}).
For each $t \in \N$, define
\begin{multline}
	G_{\pm\e,t}
	\coloneqq
	\biggl( \frac{2+U_t}{14}(1-B_{\frac{1\pm\e}{2},t}) 
	+ \frac{3+U_t}{14}B_{\frac{1\pm\e}{2},t} \biggr) \tilde{B}_t + \varphi(U_t)(1-\tilde{B}_t) \;,
	\label{e:representation_of_V-full}
\end{multline}
$V_{\pm\e,t} \coloneqq G_{\pm\e,2t-1}$, $W_{\pm\e,t} \coloneqq G_{\pm\e,2t}$, $\xi_{\pm\e,t} \coloneqq V_{\pm\e,t} - \mu_{\pm\e}$, and $\zeta_{\pm\e,t} \coloneqq W_{\pm\e,t} - \mu_{\pm\e}$.

In the following, if $a_1,\dots,a_{K^d}$ is a sequence of elements, we will use the notation $a_{1:K^d}$ as a shorthand for $(a_1,\dots,a_{K^d})$. 
For each $\e_1,\dots,\e_{K^d} \in \{-\e,\e\}$, each $i \in [K^d]$, and each $j \in [n]$, define the random variables $\xi^{\e_{1:K^d}}_{j+(i-1)n} \coloneqq \xi_{\e_i,j+(i-1)n}$ and $\zeta^{\e_{1:K^d}}_{j+(i-1)n} \coloneqq \zeta_{\e_i,j+(i-1)n}$.
One can check that
the family $\brb{\xi^{\e_{1:K^d}}_{t},\zeta^{\e_{1:K^d}}_{t}}_{t\in [T], \e_{1:K^d}\in\{- \e , \e\}^{K^d}}$ is an independent family, and for each $i \in [K^d]$ and each $j \in [n]$ the two random variables $\xi^{\e_{1:K^d}}_{j+(i-1)n},\zeta^{\e_{1:K^d}}_{j+(i-1)n}$ are zero mean with common distribution given by a shift by $\mu_{\e_i}$ of $\cD_{\e_i}$.
For each $\e_1,\dots,\e_{K^d} \in \{-\e,\e\}$, for each $i \in [K^d]$ and $j \in [n]$, let $V^{\e_{1:{K^d}}}_{j+(i-1)n} \coloneqq \mu_{\e_i} + \xi^{\e_{1:K^d}}_{j+(i-1)n}$ and $W^{\e_{1:K^d}}_{j+(i-1)n} \coloneqq \mu_{\e_i} + \zeta^{\e_{1:K^d}}_{j+(i-1)n}$.
Note that these last two random variables are $[0,1]$-valued zero-mean perturbations of $\mu_{\e_i}$ with shared density given by $f_{\e_i}$, and hence bounded by $2$.

Crucially, we assume that the learner knows what will be the sequence of the contexts in advance, and hence we can restrict the proof to deterministic algorithms without any loss of generality.
Specifically, define, for all $(i_1, \dots, i_d)\in[K]^d$ and $j\in[n]$, the lattice points
\[
\bx_{i_1,\ldots,i_d, \, j}
\ceq
\left(  
\frac{i_1-1}{K}
,
\ldots 
, 
\frac{i_d-1}{K}
\right).
\]
Then, define the contexts $(\bx_t)_{t\in[T]}$ as the contexts corresponding to the lexicographic (increasing) ordering of the indices of $\brb{\bx_{i_1,\ldots,i_d, j}}_{(i_1,\ldots,i_d,j)\in[K]^d\times[n]}$, where vectors of indices $(i_1,\ldots,i_d,j)$ are thought of as digits composing a numerical string. 
In words, the first $n$ contexts $\bx_1, \dots, \bx_n$ are all equal to $\bx_{1,\ldots,1, 1} = \dots = \bx_{1,\ldots,1, n} = (0,\ldots,0)$, the next $n$ contexts $\bx_{n+1}, \dots, \bx_{2n}$ are $\bx_{1,\ldots,1,2, 1} = \bx_{1,\ldots,1,2, n} = (0,\ldots,1/K)$, and so on, until the last $n$ contexts $\bx_{T-n+1} = \dots = \bx_T$ that are all equal to $\brb{ \frac{K-1}{K},\ldots,\frac{K-1}{K} }$. 
Now, notice that since $\e = 1/K$ then for any $\e_{1:K^d} \in \{-\e,\e\}^{K^d}$ we have that the valuations $V^{\e_{1:K^d}}_1,W^{\e_{1:K^d}}_1,\dots,V^{\e_{1:K^d}}_T,W^{\e_{1:K^d}}_T$ are consistent with our model, in the sense that they are zero-mean noisy perturbations of the market values $\mu_t$ defined by paramerizing every $t\in[T]$ by $ t = j+(i-1)n $, with $(i,j)\in[K^d]\times[n]$, and letting $\mu_t \ceq \mu_{\e_i}$, and these market values and contexts satisfy \Cref{i:two} in our model definition.

We will show that for each algorithm for contextual brokerage with full feedback and each time horizon $T$, if $R_T^{\e_{1:K^d}}$ is the regret of the algorithm at time horizon $T$ when the traders' valuations are $V^{\e_{1:K^d}}_1,W^{\e_{1:K^d}}_1,\dots,V^{\e_{1:K^d}}_T,W^{\e_{1:K^d}}_T$, then $\max_{\e_{1:K^d} \in \{-\e,\e\}^{K^d}} R_T^{\e_{1:K^d}} = \Omega\brb{T^{\frac{d}{d+2}} }$ with our choices of $\e$ and $K$.

To do it, we first denote, for any $\e_1,\dots,\e_{K^d}\in\{-\e,\e\}$, $p\in[0,1]$, and $t \in [T]$, $\GFT^{\e_{1:K^d}}_{t}(p) \coloneqq \gft(p,V^{\e_{1:K^d}}_{t}, W^{\e_{1:K^d}}_{t})$.

Then, by \Cref{l:structural}, we have, for all $\e_1,\dots,\e_{K^d} \in \{-\e,\e\},i \in [K^d],j\in[n]$, and $p \in [0,1]$,
\begin{multline*}
	\E\bsb{\GFT^{\e_{1:K^d}}_{j+(i-1)n}(p) }
	=
	2\int_0^p\int_0^\lambda f_{\e_i}(s)\dif s \dif \lambda 
	+ 2 (\mu_{\e_i} - p)\int_0^p f_{\e_i}(s) \dif s \;,
\end{multline*}
which,  together with the fundamental theorem of calculus ---\cite[Theorem 14.16]{bass2013real}, noting that $p\mapsto\E\bsb{ \GFT^{\e_{1:K^d}}_{j+(i-1)n}(p) }$ is absolutely continuous with derivative defined a.e.\ by $p\mapsto 2(\mu_{\e_i}-p) f_{\e_i}(p)$--- yields, for any $p\in[2/7,1]$,
\begin{equation}
	\E\bsb{ \GFT^{\e_{1:K^d}}_{j+(i-1)n}(\mu_{\e_i}) } - \E\bsb{ \GFT^{\e_{1:K^d}}_{j+(i-1)n}(p) }
	=
	|\mu_{\e_i} - p|^2 \;.
	\label{e:proof-lb-full}
\end{equation}

Hence note that for all $\e_{1:K^d} \in \{-\e,\e\}^{K^d}$, $i \in [K^d]$, $j \in [n]$, and $p<\frac12$, if $\e_i > 0$, a direct verification shows that
\begin{equation}
	\E\lsb{ \GFT^{\e_{1:K^d}}_{j+(i-1)n} \lrb{\fracc12} }
	\ge
	\E\bsb{ \GFT^{\e_{1:K^d}}_{j+(i-1)n} (p) }
	+ \Omega(\e^2)\;.
	\label{e:lower-bound-1-full}
\end{equation}
Similarly, for all $\e_{1:K^d} \in \{-\e,\e\}^{K^d}$, $i \in [K^d]$, $j \in [n]$, and $p>\frac12$, if $\e_i < 0$, then
\begin{equation}
	\E\lsb{ \GFT^{\e_{1:K^d}}_{j+(i-1)n} \lrb{\fracc12} }
	\ge
	\E\bsb{ \GFT^{\e_{1:K^d}}_{j+(i-1)n} (p) }
	+ \Omega(\e^2)
	\;.
	\label{e:lower-bound-2-full}
\end{equation}

In words, in the $\e_i = \e$ (resp., $\e_i = -\e$) case, the optimal price for the rounds $1+(i-1)n,\dots, in$ belongs to the region $\bigl(\frac{1}{2},1\bigr]$ (resp., $\bigl[0, \frac12\bigr)$).
Then, with a standard information-theoretic argument, since $\Omega(1/\e^2) = \Omega(n)$ samples are needed to determine the sign of $\e_i$ (with high probability), any algorithm will pay a regret of at least $\Omega\brb{ \frac{1}{\e^2} \cdot \e^2} = \Omega(1)$ for each point in the lattice, for a total regret of $\Omega( 1 \cdot K^d) = \Omega\brb{ T^{d/(d+2)} }$.

\section{PROOF OF THEOREM \ref{t:upper-bound-limited-main}}
\label{section:simplified:twobitfeed:oned}

Note first that the decisions to either explore or exploit (\Cref{state:if:exploit,if:explore:and:bisect}), the decisions to whether or not bisect (\Cref{state:bisect}) during a run of \exbis{}, and $\cC_t, i_t, m_t, \cE_t, n_t, \cC'_t, q_t, \cE'_t, n'_t$ (for all $t\in[T]$) are deterministic, since they are only determined by the adversarial sequence of contexts $\bx_1,\ldots,\bx_T$.

Hence, for any time step $t\in[T]$ where the property $n_t \ge n'_t$ on \Cref{state:if:first:exploit} holds (which, again, is a deterministic event), we have that $n_t \ge n'_t \ge 2^{2 (i_t-1)}$, since the parent cell $\cC'_t$ of $\cC_t$ was bisected on \Cref{state:bisect}, with \Cref{if:explore:and:bisect} holding.
Proceeding as in \Cref{section:simplified:fullfeed}, this implies that, for any time step $t\in[T]$ where the property $n_t \ge n'_t$ on \Cref{state:if:first:exploit} holds, the instantaneous regret of \exbis{} satisfies
\begin{multline*}
	\sup_{p_t\in[0,1]} \E \bsb{ \gft ( p_t, V_t, W_t ) } - \E \bsb{ \gft ( P_t, V_t, W_t ) }
	\le 
	M 2^{-2 i_t}
	+
	\frac{M}{2} 2^{-2 (i_t-1)}
	= 
	3 M \cdot 2^{-2 i_t}.
\end{multline*}

Proceeding once more as in \Cref{section:simplified:fullfeed}, we obtain that for all time steps $t\in[T]$ such that property $n_t < n'_t$ on \Cref{state:if:second:exploit} holds, the instantaneous regret of \exbis{} satisfies
\begin{multline*}
	\sup_{p_t\in[0,1]} \E \bsb{ \gft ( p_t, V_t, W_t ) } - \E \bsb{ \gft ( P_t, V_t, W_t ) }
	\le 
	M 2^{-2 (i_t-1)}
	+
	\frac{M}{2} 2^{-2 (i_t-1)}
	= 6M \cdot 2^{-2 i_t}.
\end{multline*}
Therefore, the instantaneous regret of \exbis{} in \emph{all} exploiting rounds $t\in[T]$ such that $m_t < 2^{4i_t}$ on \Cref{state:if:exploit} satisfies
\begin{equation}
	\label{e:upper-bound-exploit}
	\sup_{p_t\in[0,1]} \E \bsb{ \gft ( p_t, V_t, W_t ) } - \E \bsb{ \gft ( P_t, V_t, W_t ) }
	\le
	6M \cdot 2^{-2 i_t}.
\end{equation}
Now, for any cell $\cC$ of level $i_\cC$ obtained by successive bisections of $[0,1)^d$, let $\cF_\cC$ be the set of all time steps $s \in [T]$ for which $\cC_t=\cC$ on \Cref{state:set:Ct:defined},
$n_{\cC} \ceq \labs{\cF_\cC}$, 
%
%
and note that the sum of the instantaneous regrets of \exbis{} over time steps $t\in\cF_\cC$ such that $m_t \ge 2^{4i_t}$ on \Cref{if:explore:and:bisect} (i.e., total regret due to \emph{exploration} in rounds $t$ where $\cC_t = \cC$) 
is less than or equal to the sum of the upper bounds in \Cref{e:upper-bound-exploit} of the instantaneous regrets of \exbis{} over time steps $t\in\cF_\cC$ such that $m_t < 2^{4i_t}$ on \Cref{state:if:exploit} (i.e., total regret due to \emph{exploitation} in rounds $t$ where $\cC_t = \cC$), since each exploration time $t\in\cF_\cC$ is preceded by $2^{4 i_\cC}$ exploitation times (by \Cref{state:Mt,state:if:exploit}), yielding a total regret due to exploitation of $6M \cdot 2^{4 i_\cC} \cdot 2^{-2 i_\cC} = 6M \cdot 2^{2 i_\cC}$, whereas the maximal number of exploration times is $2^{2 i_\cC}$ (by \Cref{if:explore:and:bisect,state:bisect}). 

Hence, it is sufficient to control the sum over all exploitation rounds of the upper bounds in \Cref{e:upper-bound-exploit} to obtain an upper bound on the regret of \exbis{}, up to a factor of $2$. 
With this in mind, and with the same notation $(\cC_{i,j})_{i\in\{0,1,\dots\}, j\in[2^{di}]}$ introduced in \Cref{section:simplified:fullfeed} for the family of all cells that can be obtained by successive bisections of $[0,1)^d$,
let $\widetilde{R}_T$ be the upper bound on the total regret due to exploitation defined as the sum over all exploitation rounds (i.e., over all rounds $t\in[T]$ such that $m_t < 2^{4i_t}$ on \Cref{state:if:exploit})  of the upper bounds in \Cref{e:upper-bound-exploit}. 
We have
\begin{equation}
	\label{e:upper-bound-total-exploit}
	\widetilde{R}_T
	\le 
	6M
	\sum_{i=0}^{\infty}
	\sum_{j=1}^{2^{id}}
	n_{\cC_{i,j}}
	2^{-2i},
\end{equation}
with the constraints that (letting $\N_0 \ceq \{0,1,2,\dots\}$):
\begin{equation}
	\label{e:constraints-exploit}
	n_{\cC_{i,j}} \le 2^{4 i},
	\ \forall i \in \N_0, \forall j \in [2^{id}],
	\quad\text{and}\quad
	\sum_{i=0}^{\infty}
	\sum_{j=1}^{2^{id}} n_{\cC_{i,j}}
	\le
	T.
\end{equation}
Let $k$ be the smallest integer such that 
\[
\sum_{i=1}^k 
2^{id}
2^{4i} \ge T.
\]
Then
\[
\frac{1 - 2^{(k+1)(d+4)}}{1 - 2^{d+4}}
\ge T
\]
or, equivalently,
\[
2^{(k+1)(d+4)}
\ge 
1 + (2^{d+4}-1)T.
\]
By definition of $k$, we also have
\[
\sum_{i=1}^{k-1} 
2^{id}
2^{4i} < T.
\]
Then
\[
\frac{1 - 2^{k(d+4)}}{1 - 2^{d+4}}
< T
\]
or, equivalently,
\[
2^{k(d+4)}
< 
1 + (2^{d+4}-1)T.
\]
By the constraints in \Cref{e:constraints-exploit} and the definition of $k$, \Cref{e:upper-bound-total-exploit} implies that
\begin{align*}
	\frac{
		\widetilde{R}_T
	}{6M}
	\le &
	\sum_{i=0}^k
	2^{id} 2^{4i} 2^{-2i}
	\\
	= &
	\frac{1 - 2^{(k+1)(d+2)}}{1-2^{d+2}}
	\\ 
	\le &
	\frac{1}{2^{d+2}-1}
	2^{(k+1)(d+2)}
	\\ 
	= &
	\frac{1}{2^{d+2}-1}
	2^{k(d+4)}
	2^{d+2-2k} 
	\\ 
	\le &
	\frac{1+(2^{d+4}-1)T}{2^{d+2}-1} 
	2^{d+2-2k}.
\end{align*}
Now, recalling that
\[
2^{(k+1)(d+4)}
\ge 
1 + (2^{d+4}-1)T,
\]
we obtain
\[
2^{k(d+4)}
\ge 
\frac{ (2^{d+4}-1)
}{
	2^{d+4}
}
T
\]
and consequently
\[
2^{k}
\ge 
\left(
\frac{ 2^{d+4}-1
}{
	2^{d+4}
}
\right)^{\frac{1}{d+4}}
T^{\frac{1}{d+4}}.
\]

Hence 
\begin{align*} 
	\frac{
		\widetilde{R}_T
	}{6M}
	&
	\le
	\frac{
		2^{d+2} \left(1+(2^{d+4}-1) \right)}{2^{d+2}-1} 
	T
	2^{-2k}
	\\
	&
	\le 
	\frac{
		\frac{
			2^{d+2} \left(1+(2^{d+4}-1) \right)}{2^{d+2}-1}
	}{
		\left(
		\frac{ 2^{d+4}-1
		}{
			2^{d+4}
		}
		\right)^{\frac{2}{d+4}}
	}
	\cdot
	T^{1-\frac{2}{d+4}}.
	\\
	& =
	\frac{
		\frac{
			2^{d+2} \left(1+(2^{d+4}-1) \right)}{2^{d+2}-1}
	}{
		\left(
		\frac{ 2^{d+4}-1
		}{
			2^{d+4}
		}
		\right)^{\frac{2}{d+4}}
	}
	\cdot
	T^{\frac{d+2}{d+4}}
	\\
	&
	=
	\frac{\frac{2^{d+2}2^{d+4}}{2^{d+2}-1}}{\left(1-\frac{1}{2^{d+4}}\right)^{\frac{2}{d+4}}}
	\cdot
	T^{\frac{d+2}{d+4}}
	\\
	&
	=
	2^{4}\cdot2^{d}\cdot\frac{\frac{2^{d+2}}{2^{d+2}-1}}{\left(1-\frac{1}{2^{d+4}}\right)^{\frac{2}{d+4}}}
	\cdot
	T^{\frac{d+2}{d+4}}
	\\
	&
	\le
	2^{4}\cdot2^{d}\cdot\frac{\frac{2^{1+2}}{2^{1+2}-1}}{\left(1-\frac{1}{2^{1+4}}\right)^{\frac{2}{1+4}}}
	\cdot
	T^{\frac{d+2}{d+4}}
	\\
	&
	=
	2^{4}\cdot2^{d}\cdot\frac{32}{7\cdot31^{2/5}}
	\cdot
	T^{\frac{d+2}{d+4}}
\end{align*}
where, in the last inequality, we used the fact that the function $d\mapsto \fracc{\frac{2^{d+2}}{2^{d+2}-1}}{\left(1-\frac{1}{2^{d+4}}\right)^{\frac{2}{d+4}}}$ is decreasing on $d\ge 1$.
This immediately implies that
\[
R_T
\le
2 \widetilde{R}_T
\le
223 \cdot M \cdot 2^d  \cdot T^{\frac{d+2}{d+4}}
=
O \lrb{ T^{\frac{d+2}{d+4}} }.
\]

\section{PROOF OF THEOREM \ref{t:lower-bound-limited-main}}
\label{s:proof-lower-bound-limited}

The proof shares a similar construction as the proof of \Cref{t:lower-bound-full-main}, with a different tuning and additional arguments at the end to account for the limited feedback. 
Nevertheless, we give the full details here for ease of exposition.

Fix $T\in \N$. 
Assume without loss of generality that $K \ceq T^{\frac{1}{d+4}}$ is an integer, and note that $K^d$ divides $T$.
Let $n \ceq \frac{T}{K^d} = T^{4/(d+4)} = K^4 \in \N$ and $\e \ceq n^{-1/4} = T^{-1/(d+4)}$.
Let
$
f_{\pm \e} 
\coloneqq
1
\mp \e \I_{\lsb{\frac17,\frac3{14}}} 
\pm \e \I_{ \left( \frac3{14},\frac27 \right] }
$.
Note that $0 \le f_{\pm\e} \le 2$ and $\int_0^1 f_{\pm\e}(x) \dif x = 1$, hence $f_{\pm\e}$ is a valid density on $[0,1]$ bounded by $M=2$.
We will denote the corresponding probability measure by $\cD_{\pm\e}$ and define $\mu_{\pm\e} \coloneqq \int_{[0,1]} x \dif \cD_{\pm\e}(x) = \frac12 \pm \frac{\e}{196}$.
Consider for each $q \in [0,1]$, an i.i.d.\ sequence $(B_{q,t})_{t \in \N}$ of Bernoulli random variables of parameter $q$, an i.i.d.\ sequence $(\tilde{B}_t)_{t \in \N}$ of Bernoulli random variables of parameter $1/7$, an i.i.d.\ sequence $(U_t)_{t \in \N}$ of uniform random variables on $[0,1]$, such that $\lrb{(B_{q,t})_{t \in \N, q \in [0,1]} , (\tilde{B}_t)_{t \in \N}, (U_t)_{t \in \N}}$ is an independent family. 
Let $\varphi \colon [0,1] \to [0,1]$ be such that, if $U$ is a uniform random variable on $[0,1]$, then the distribution of $\varphi(U)$ has density $\frac{7}{6} \cdot \I_{[0,1]\m [\fracc{1}{7},\fracc{2}{7}]}$ (which exists by the Skorokhod representation theorem, \citealt[Section 17.3]{williams1991probability}).
For each $t \in \N$, define
\begin{multline}
	G_{\pm\e,t}
	\coloneqq
	\biggl( \frac{2+U_t}{14}(1-B_{\frac{1\pm\e}{2},t}) 
	+ \frac{3+U_t}{14}B_{\frac{1\pm\e}{2},t} \biggr) \tilde{B}_t + \varphi(U_t)(1-\tilde{B}_t) \;,
	\label{e:representation_of_V}
\end{multline}
$V_{\pm\e,t} \coloneqq G_{\pm\e,2t-1}$, $W_{\pm\e,t} \coloneqq G_{\pm\e,2t}$, $\xi_{\pm\e,t} \coloneqq V_{\pm\e,t} - \mu_{\pm\e}$, and $\zeta_{\pm\e,t} \coloneqq W_{\pm\e,t} - \mu_{\pm\e}$.

In the following, if $a_1,\dots,a_{K^d}$ is a sequence of elements, we will use the notation $a_{1:K^d}$ as a shorthand for $(a_1,\dots,a_{K^d})$. 
For each $\e_1,\dots,\e_{K^d} \in \{-\e,\e\}$, each $i \in [K^d]$, and each $j \in [n]$, define the random variables $\xi^{\e_{1:K^d}}_{j+(i-1)n} \coloneqq \xi_{\e_i,j+(i-1)n}$ and $\zeta^{\e_{1:K^d}}_{j+(i-1)n} \coloneqq \zeta_{\e_i,j+(i-1)n}$.
One can check that
the family $\brb{\xi^{\e_{1:K^d}}_{t},\zeta^{\e_{1:K^d}}_{t}}_{t\in [T], \e_{1:K^d}\in\{- \e , \e\}^{K^d}}$ is an independent family, and for each $i \in [K^d]$ and each $j \in [n]$ the two random variables $\xi^{\e_{1:K^d}}_{j+(i-1)n},\zeta^{\e_{1:K^d}}_{j+(i-1)n}$ are zero mean with common distribution given by a shift by $\mu_{\e_i}$ of $\cD_{\e_i}$.
For each $\e_1,\dots,\e_{K^d} \in \{-\e,\e\}$, for each $i \in [K^d]$ and $j \in [n]$, let $V^{\e_{1:{K^d}}}_{j+(i-1)n} \coloneqq \mu_{\e_i} + \xi^{\e_{1:K^d}}_{j+(i-1)n}$ and $W^{\e_{1:K^d}}_{j+(i-1)n} \coloneqq \mu_{\e_i} + \zeta^{\e_{1:K^d}}_{j+(i-1)n}$.
Note that these last two random variables are $[0,1]$-valued zero-mean perturbations of $\mu_{\e_i}$ with shared density given by $f_{\e_i}$, and hence bounded by $2$.

Crucially, we assume that the learner knows what will be the sequence of the contexts in advance, and hence we can restrict the proof to deterministic algorithms without any loss of generality.
Specifically, 
define, for all $(i_1, \dots, i_d)\in[K]^d$ and $j\in[n]$, the lattice points
\[
\bx_{i_1,\ldots,i_d, \, j}
\ceq
\left(  
\frac{i_1-1}{K}
,
\ldots 
, 
\frac{i_d-1}{K}
\right).
\]
Then, define the contexts $(\bx_t)_{t\in[T]}$ as the contexts corresponding to the lexicographic (increasing) ordering of the indices of $\brb{\bx_{i_1,\ldots,i_d, j}}_{(i_1,\ldots,i_d,j)\in[K]^d\times[n]}$, where vectors of indices $(i_1,\ldots,i_d,j)$ are thought of as digits composing a numerical string. 
In words, the first $n$ contexts $\bx_1, \dots, \bx_n$ are all equal to $\bx_{1,\ldots,1, 1} = \dots = \bx_{1,\ldots,1, n} = (0,\ldots,0)$, the next $n$ contexts $\bx_{n+1}, \dots, \bx_{2n}$ are $\bx_{1,\ldots,1,2, 1} = \bx_{1,\ldots,1,2, n} = (0,\ldots,1/K)$, and so on, until the last $n$ contexts $\bx_{T-n+1} = \dots = \bx_T$ that are all equal to $\brb{ \frac{K-1}{K},\ldots,\frac{K-1}{K} }$. 
%

Now, notice that since $\e = 1/K$ then for any $\e_{1:K^d} \in \{-\e,\e\}^{K^d}$ we have that the valuations $V^{\e_{1:K^d}}_1,W^{\e_{1:K^d}}_1,\dots,V^{\e_{1:K^d}}_T,W^{\e_{1:K^d}}_T$ are consistent with our model, in the sense that they are zero-mean noisy perturbations of the market values $\mu_t$ defined by paramerizing every $t\in[T]$ by $ t = j+(i-1)n $, with $(i,j)\in[K^d]\times[n]$, and letting $\mu_t \ceq \mu_{\e_i}$, and these market values and contexts satisfy \Cref{i:two} in our model definition.

We will show that for each algorithm for contextual brokerage with limited feedback and each time horizon $T$, if $R_T^{\e_{1:K^d}}$ is the regret of the algorithm at time horizon $T$ when the traders' valuations are $V^{\e_{1:K^d}}_1,W^{\e_{1:K^d}}_1,\dots,V^{\e_{1:K^d}}_T,W^{\e_{1:K^d}}_T$, then $\max_{\e_{1:K^d} \in \{-\e,\e\}^{K^d}} R_T^{\e_{1:K^d}} = \Omega\brb{T^{\frac{d+2}{d+4}} }$ with our choices of $\e$ and $K$.

To do it, we first denote, for any $\e_1,\dots,\e_{K^d}\in\{-\e,\e\}$, $p\in[0,1]$, and $t \in [T]$, $\GFT^{\e_{1:K^d}}_{t}(p) \coloneqq \gft(p,V^{\e_{1:K^d}}_{t}, W^{\e_{1:K^d}}_{t})$.

Then, by \Cref{l:structural}, we have, for all $\e_1,\dots,\e_{K^d} \in \{-\e,\e\},i \in [K^d],j\in[n]$, and $p \in [0,1]$,
\begin{multline*}
	\E\bsb{\GFT^{\e_{1:K^d}}_{j+(i-1)n}(p) }
	=
	2\int_0^p\int_0^\lambda f_{\e_i}(s)\dif s \dif \lambda 
	+ 2 (\mu_{\e_i} - p)\int_0^p f_{\e_i}(s) \dif s \;,
\end{multline*}
which,  together with the fundamental theorem of calculus ---\cite[Theorem 14.16]{bass2013real}, noting that $p\mapsto\E\bsb{ \GFT^{\e_{1:K^d}}_{j+(i-1)n}(p) }$ is absolutely continuous with derivative defined a.e.\ by $p\mapsto 2(\mu_{\e_i}-p) f_{\e_i}(p)$--- yields, for any $p\in[2/7,1]$,
\begin{equation}
	\E\bsb{ \GFT^{\e_{1:K^d}}_{j+(i-1)n}(\mu_{\e_i}) } - \E\bsb{ \GFT^{\e_{1:K^d}}_{j+(i-1)n}(p) }
	=
	|\mu_{\e_i} - p|^2 \;.
	\label{e:proof-lb}
\end{equation}

Hence note that for all $\e_{1:K^d} \in \{-\e,\e\}^{K^d}$, $i \in [K^d]$, $j \in [n]$, and $p<\frac12$, if $\e_i > 0$, a direct verification shows that
\begin{equation}
	\E\lsb{ \GFT^{\e_{1:K^d}}_{j+(i-1)n} \lrb{\fracc12} }
	\ge
	\E\bsb{ \GFT^{\e_{1:K^d}}_{j+(i-1)n} (p) }
	+ \Omega(\e^2)\;.
	\label{e:lower-bound-1}
\end{equation}
Similarly, for all $\e_{1:K^d} \in \{-\e,\e\}^{K^d}$, $i \in [K^d]$, $j \in [n]$, and $p>\frac12$, if $\e_i < 0$, then
\begin{equation}
	\E\lsb{ \GFT^{\e_{1:K^d}}_{j+(i-1)n} \lrb{\fracc12} }
	\ge
	\E\bsb{ \GFT^{\e_{1:K^d}}_{j+(i-1)n} (p) }
	+ \Omega(\e^2)
	\;.
	\label{e:lower-bound-2}
\end{equation}
Furthermore, a direct verification shows that, for each $\e_{1:K^d} \in \{-\e,\e\}^{K^d}$ and $t \in [T]$, 
\begin{multline}
	\max_{p \in [0,1]} \E\bsb{ \GFT^{\e_{1:K^d}}_{t} (p)} - \max_{p \in [\frac{1}{7},\frac{2}{7}]} \E\bsb{ \GFT^{\e_{1:K^d}}_{t} (p)} 
	\ge \frac{1}{50} 
	= \Omega(1) \;.
	\label{e:lower-bound-3}
\end{multline}
%
In words, in the $\e_i = \e$ (resp., $\e_i = -\e$) case, the optimal price for the rounds $1+(i-1)n,\dots, in$ belongs to the region $\bigl(\frac{1}{2},1\bigr]$ (resp., $\bigl[0, \frac12\bigr)$).
By posting prices in the wrong region $\bsb{0, \frac{1}{2}}$ (resp., $[\frac{1}{2}, 1)$) in the $\e_i = \e$ (resp., $\e_i = -\e$) case, the learner incurs a $\Omega(\e^2) = \Omega\brb{1/K^2}$ instantaneous regret by  \eqref{e:lower-bound-1} (resp.,  \eqref{e:lower-bound-2}).
Then, in order to attempt suffering less than $\Omega\brb{1/K^2 \cdot n} = \Omega(1/\e^2)$ cumulative regret in the rounds $1+(i-1)n,\dots, in$, the algorithm would have to detect the sign of $\e_i$ and play accordingly.
We will show now that even this strategy will not improve the regret of the algorithm (by more than a constant) because of the cost of determining the sign of $\e_i$ with the available feedback. 
Since for any $i \in [K^d]$ and $j \in [n]$, the feedback received from the two traders at time $j+(i-1)n$ by posting a price $p$ is $\I\bcb{p \le V^{\e_{1:K^d}}_{j+(i-1)n}}$ and $\I\bcb{ p \le W^{\e_{1:K^d}}_{j+(i-1)n} }$, the only way to obtain information about (the sign of) $\e_i$ is to post in the costly ($\Omega(1)$-instantaneous regret by \cref{e:lower-bound-3}) sub-optimal region $[\frac{1}{7},\frac{2}{7}]$
.
However, posting prices in the region $[\frac{1}{7},\frac{2}{7}]$ at time $j + (i-1)n$ can't give more information about (the sign of) $\e_i$ than the information carried by $V^{\e_{1:K^d}}_{j+(i-1)n}$ and $W^{\e_{1:K^d}}_{j+(i-1)n}$, which, in turn, can't give more information about (the sign of) $\e_i$ than the information carried by the two Bernoullis $B_{\frac{1+\e_i}{2},2(j+(i-1)n)-1}$ and $B_{\frac{1+\e_i}{2},2(j+(i-1)n)}$. 
Also, notice that only during rounds $1+(i-1)n,\dots,in$ it is possible to extract information about the sign of $\e_i$.
With a standard information-theoretic argument, in order to distinguish the sign of $\e_i$ having access to i.i.d.\ Bernoulli random variables of parameter $\frac{1+\e_i}{2}$ requires $\Omega(1/\e^2) $ samples. Hence we are forced to post at least $\Omega( 1/\e^2 )$ prices in the costly region $\bsb{\frac{1}{7},\frac{2}{7}}$ during the rounds $1+(i-1)n,\dots,in$ suffering a regret of $\Omega( 1/\e^2 ) \cdot \Omega(1) = \Omega( 1/\e^2 )$.
Putting everything together, no matter what the strategy, each algorithm will pay at least $\Omega\brb{ 1/\e^2}$ regret in each epoch $1+(i-1)n,\dots,in$ for every $i \in [K^d]$, resulting in an overall regret of $K^d \cdot \Omega( 1/\e^2) = \Omega\brb{ T^{\frac{d+2}{d+4}}}$.

\section{MISSING DETAILS FROM SECTION \ref{s:one-half-approx}}
\label{s:approximation-appe}

In the proof of our approximation theorem, we leverage the following result from \cite[Lemma 1]{bachoc2024contextual}.
\begin{lemma}
	\label{l:structural}
	Suppose that $V$ and $W$ are two $[0,1]$-valued independent random variables with possibly different densities but both bounded by the same constant $M\ge 1$, and such that $\E[V] = \E[W] \eqqcolon \mu$. 
	Denote by $F$ (resp., $G$) the cumulative distribution function of $V$ (resp., $W$).
	Then, for each $p \in[0,1]$, it holds that
	\[
	\E\bsb{\gft(p,V,W)}
	=
	\int_0^p (F+G)(\lambda) \dif \lambda + (\mu-p)(F+G)(p)
	\]
	and 
	\[
	0 \le \E\bsb{\gft(\mu,V,W) - \gft(p,V,W) } \le M \labs{\mu-p}^2 \;.
	\]
\end{lemma}

We now prove the optimality of our $\frac{1}{2}$-approximation result.


\begin{proof}[Proof of Theorem \ref{t:approximation-tight}]
	
	Fix $\delta \in (0,1/6)$ and the probability density functions:
	\begin{align*}
		f_\delta \colon [0,1] & \to \bbR , 
		\\
		x  & \mapsto f_\delta (x) \ceq
		\frac{1}{2\delta} \I\{ 0 \le x \le \delta \} + \frac{1}{2\delta} \I\{ 1-\delta \le x \le 1 \} \;
	\end{align*}
	and
	\begin{align*}
		g_\delta  \colon [0,1] & \to \bbR , 
		\\
		x & \mapsto g_\delta (x) \ceq
		\frac{1}{2\delta} \I\lcb{ \frac{1}{2}-\delta \le x \le \frac{1}{2}+\delta } \;.
	\end{align*}
	Let $V_\delta$ and $W_\delta$ be two independent random variables with probability density functions $f_\delta$ and $g_\delta$, respectively.
	Then:
	{\allowdisplaybreaks
		\begin{align*}
			\allowdisplaybreaks
			&   \E \bsb{ \labs{V_\delta-W_\delta} } \\
			= &
			\int_{[0,1]^2} \labs{ v-w } f_\delta(v) g_\delta(w) \dif v \dif w
			\\
			= &
			\frac{1}{4\delta^2} \int_{[0,\delta]\cup[1-\delta,1]} \lrb{ \int_{\lsb{\frac12-\delta,\frac12+\delta}}\labs{ v-w } \dif w }  \dif v 
			\\
			= &
			\frac{1}{4\delta^2} \int_0^\delta \lrb{ \int_{\frac12-\delta}^{\frac12+\delta} ( w-v ) \dif w }  \dif v +
			\frac{1}{4\delta^2} \int_{1-\delta}^1 \lrb{ \int_{\frac12-\delta}^{\frac12+\delta} ( v-w ) \dif w }  \dif v
			\\
			= &
			\frac{1}{4\delta^2} \int_0^\delta \lsb{ \frac{( w-v )^2}{2}} _{w = \frac12-\delta}^{\frac12+\delta}  \dif v
		       +
			\frac{1}{4\delta^2} \int_{1-\delta}^1 \lsb{ -\frac{( v-w )^2}{2}  }_{w = \frac12-\delta}^{\frac12+\delta}  \dif v
			\\
			= &
			\frac{1}{4\delta^2} \int_0^\delta \lrb{ \frac{( \frac12+\delta-v )^2}{2} - \frac{( \frac12-\delta-v )^2}{2}  }  \dif v
		      +
			\frac{1}{4\delta^2} \int_{1-\delta}^1 \lrb{ -\frac{ \brb{ v-(\frac12+\delta) }^2}{2} + \frac{\brb{ v-(\frac12-\delta) }^2}{2} }  \dif v
			\\
			= &
			\frac{1}{4\delta^2} \int_0^\delta (1-2v)\delta \dif v
			+
			\frac{1}{4\delta^2} \int_{1-\delta}^1 (2v-1)\delta \dif v
			\\ 
			= &
			\frac{1}{4\delta} \int_0^\delta (1-2v) \dif v
			+
			\frac{1}{4\delta} \int_{1-\delta}^1 (2v-1) \dif v
			\\ 
			= &
			\frac{1-\delta}{2}
		\end{align*}%
	}%
	Instead, the reward accrued when posting the expectation $\mu \ceq \E[V_\delta] = \E[W_\delta] = \frac12$ is:
	{\allowdisplaybreaks
		\begin{align*}
			    \E\bsb{\gft(\mu,V_\delta,W_\delta)}   
			= &
			\int_{[0,1]^2} \lrb{v \vee w - v \wedge w} \I \{v \wedge w \le \mu \le v \vee w\} f_\delta(v) g_\delta(w) \dif v \dif w
			\\
			= &
			\frac{1}{4\delta^2} \int_{[\frac{1}{2} -\delta, \frac{1}{2} +\delta]} 
		 \lrb{\int_{[0,\delta]\cup[1-\delta,1]} ( v \vee w - v \wedge w ) \I \{v \wedge w \le \mu \le v \vee w\} \dif v } \dif w
			\\
			= &
			\frac{1}{4\delta^2} \int_{[\frac{1}{2} -\delta, \frac{1}{2} +\delta]} \Bigg(
			\int_{[0,\delta]} ( w - v ) \I \{v \le \mu \le w\} \dif v 
			+
			\int_{[1-\delta,1]} ( v - w ) \I \{ w \le \mu \le v \} \dif v 
			\Bigg) \dif w
			\\
			= &
			\frac{1}{4\delta^2} \int_{[\frac{1}{2} -\delta, \frac{1}{2} +\delta]} \Bigg(
			\int_{[0,\delta]} ( w - v ) \I \{\mu \le w\} \dif v 
		 +
			\int_{[1-\delta,1]} ( v - w ) \I \{ w \le \mu \} \dif v 
			\Bigg) \dif w
			\\
			= &
			\frac{1}{4\delta^2} \int_{[\frac{1}{2}, \frac{1}{2} +\delta]} \lrb{
				\int_{[0,\delta]} ( w - v )  \dif v 
			} \dif w
		        +
			\frac{1}{4\delta^2} \int_{[\frac{1}{2} -\delta, \frac{1}{2} ]} \lrb{
				\int_{[1-\delta,1]} ( v - w )  \dif v 
			} \dif w
			\\
			= &
			\frac{1}{4\delta^2} \int_{[\frac{1}{2}, \frac{1}{2} +\delta]} \lrb{
				\frac{w^2}{2}
				-\frac{(w - \delta)^2}{2}
			} \dif w
			        +
			\frac{1}{4\delta^2} \int_{[\frac{1}{2} -\delta, \frac{1}{2} ]} \lrb{
				\frac{(1-w)^2}{2}
				-
				\frac{(1 - \delta - w)^2}{2}
			} \dif w
			\\
			= & \frac{1}{4} \;.
		\end{align*}%
	}%
	We will now show that, for all $\e \in (0,1/10)$, there exist two independent $[0,1]$-valued random variables $V$ and $W$ with expectation $1/2$ and admitting bounded densities such that     
	$
	\E\bsb{\gft(\mu,V,W)}
	=
	\lrb{\frac{1}{2}+\varepsilon} \cdot \E\bsb{|W-V|}
	$
	(which immediately implies that, for all $\e >0$, there exist two independent $[0,1]$-valued random variables $V$ and $W$ with bounded densities and common expectation such that $
	\E\bsb{\gft(\mu,V,W)}
	\le
	\lrb{\frac{1}{2}+\varepsilon} \cdot \E\bsb{|W-V|}
	$).
	
	Indeed, for all $\e \in (0,1/10)$, setting $\delta \coloneqq \frac{2\e}{1+2\e}$ and noting that $\delta \in (0,1/6)$, we get
	\begin{align*}
		\E\bsb{\gft(\mu,V_\delta,W_\delta)}
		&    =
		\frac{1}{4}
		=
		\lrb{\frac{1}{2}+\e}\frac{1-\delta}{2}
		\\
		&=
		\lrb{\frac{1}{2}+\e}\E \bsb{ \labs{V_\delta-W_\delta} }
		\;. \qedhere
	\end{align*}
\end{proof}

\subsubsection*{Acknowledgements}
The work of FB was supported by the Project GAP (ANR-21-CE40-0007) of the French National Research Agency (ANR) and by the Chair UQPhysAI of the Toulouse ANITI AI Cluster.
RC is partially supported by the MUR PRIN grant 2022EKNE5K (Learning in Markets and Society), the FAIR (Future Artificial Intelligence Research) project, funded by the NextGenerationEU program within the PNRR-PE-AI scheme, the EU Horizon CL4-2022-HUMAN-02 research and innovation action under grant agreement 101120237, project ELIAS (European Lighthouse of AI for Sustainability).
TC gratefully acknowledges the support of the University of Ottawa through grant GR002837 (Start-Up Funds) and that of the Natural Sciences and Engineering Research Council of Canada (NSERC) through grants RGPIN-2023-03688 (Discovery Grants Program) and DGECR2023-00208 (Discovery Grants Program, DGECR - Discovery Launch Supplement). All authors are grateful to Trent DeGiovanni for his constructive feedback.


\begin{thebibliography}{}
	
	\bibitem[Archbold et~al., 2023]{archbold2023non}
	Archbold, T., de~Keijzer, B., and Ventre, C. (2023).
	\newblock Non-obvious manipulability for single-parameter agents and bilateral
	trade.
	\newblock In {\em Proceedings of the 2023 International Conference on
		Autonomous Agents and Multiagent Systems}, pages 2107--2115, USA.
	International Foundation for Autonomous Agents and Multiagent Systems.
	
	\bibitem[Azar et~al., 2022]{azar2022alpha}
	Azar, Y., Fiat, A., and Fusco, F. (2022).
	\newblock An alpha-regret analysis of adversarial bilateral trade.
	\newblock {\em Advances in Neural Information Processing Systems},
	35:1685--1697.
	
	\bibitem[Babaioff et~al., 2020]{babaioff2020bulow}
	Babaioff, M., Goldner, K., and Gonczarowski, Y.~A. (2020).
	\newblock Bulow-klemperer-style results for welfare maximization in two-sided
	markets.
	\newblock In {\em Proceedings of the Thirty-First Annual ACM-SIAM Symposium on
		Discrete Algorithms}, SODA '20, page 2452–2471, USA. Society for Industrial
	and Applied Mathematics.
	
	\bibitem[Bachoc et~al., 2024a]{bachoc2024fair}
	Bachoc, F., Cesa-Bianchi, N., Cesari, T., and Colomboni, R. (2024a).
	\newblock Fair online bilateral trade.
	\newblock In Globerson, A., Mackey, L., Belgrave, D., Fan, A., Paquet, U.,
	Tomczak, J., and Zhang, C., editors, {\em Advances in Neural Information
		Processing Systems}, volume~37, pages 37241--37263. Curran Associates, Inc.
	
	\bibitem[Bachoc et~al., 2024b]{bachoc2024contextual}
	Bachoc, F., Cesari, T., and Colomboni, R. (2024b).
	\newblock A contextual online learning theory of brokerage.
	\newblock {\em arXiv preprint arXiv:2407.01566}.
	
	\bibitem[Bass, 2013]{bass2013real}
	Bass, R.~F. (2013).
	\newblock {\em Real analysis for graduate students}.
	\newblock Createspace Ind Pub, USA.
	
	\bibitem[Bernasconi et~al., 2024]{bernasconi2023no}
	Bernasconi, M., Castiglioni, M., Celli, A., and Fusco, F. (2024).
	\newblock No-regret learning in bilateral trade via global budget balance.
	\newblock In {\em Proceedings of the 56th Annual ACM Symposium on Theory of
		Computing}.
	
	\bibitem[Blumrosen and Mizrahi, 2016]{BlumrosenM16}
	Blumrosen, L. and Mizrahi, Y. (2016).
	\newblock Approximating gains-from-trade in bilateral trading.
	\newblock In {\em Web and Internet Economics, {WINE}'16}, volume 10123 of {\em
		Lecture Notes in Computer Science}, pages 400--413, Germany. Springer.
	
	\bibitem[Boli\'{c} et~al., 2024]{bolic2023online}
	Boli\'{c}, N., Cesari, T., and Colomboni, R. (2024).
	\newblock An online learning theory of brokerage.
	\newblock In {\em Proceedings of the 23rd International Conference on
		Autonomous Agents and Multiagent Systems}, AAMAS '24, page 216–224,
	Richland, SC. International Foundation for Autonomous Agents and Multiagent
	Systems.
	
	\bibitem[Brustle et~al., 2017]{brustle2017approximating}
	Brustle, J., Cai, Y., Wu, F., and Zhao, M. (2017).
	\newblock Approximating gains from trade in two-sided markets via simple
	mechanisms.
	\newblock In {\em Proceedings of the 2017 ACM Conference on Economics and
		Computation}, EC '17, page 589–590, New York, NY, USA. Association for
	Computing Machinery.
	
	\bibitem[Cesa-Bianchi et~al., 2023a]{cesa2023bilateral}
	Cesa-Bianchi, N., Cesari, T., Colomboni, R., Fusco, F., and Leonardi, S.
	(2023a).
	\newblock Bilateral trade: A regret minimization perspective.
	\newblock {\em Mathematics of Operations Research}.
	
	\bibitem[Cesa-Bianchi et~al., 2024]{cesa2024regret}
	Cesa-Bianchi, N., Cesari, T., Colomboni, R., Fusco, F., and Leonardi, S.
	(2024).
	\newblock Regret analysis of bilateral trade with a smoothed adversary.
	\newblock {\em Journal of Machine Learning Research}, 25(234):1--36.
	
	\bibitem[Cesa-Bianchi et~al., 2021]{cesa2021regret}
	Cesa-Bianchi, N., Cesari, T.~R., Colomboni, R., Fusco, F., and Leonardi, S.
	(2021).
	\newblock A regret analysis of bilateral trade.
	\newblock In {\em Proceedings of the 22nd ACM Conference on Economics and
		Computation}, pages 289--309, USA. Association for Computing Machinery.
	
	\bibitem[Cesa-Bianchi et~al., 2023b]{cesa2023repeated}
	Cesa-Bianchi, N., Cesari, T.~R., Colomboni, R., Fusco, F., and Leonardi, S.
	(2023b).
	\newblock Repeated bilateral trade against a smoothed adversary.
	\newblock In {\em The Thirty Sixth Annual Conference on Learning Theory}, pages
	1095--1130, USA. PMLR, PMLR.
	
	\bibitem[Cesa-Bianchi et~al., 2017]{cesa2017algorithmic}
	Cesa-Bianchi, N., Gaillard, P., Gentile, C., and Gerchinovitz, S. (2017).
	\newblock Algorithmic chaining and the role of partial feedback in online
	nonparametric learning.
	\newblock In {\em Conference on Learning Theory}, pages 465--481. PMLR.
	
	\bibitem[Cesa-Bianchi and Lugosi, 2006]{cesa2006prediction}
	Cesa-Bianchi, N. and Lugosi, G. (2006).
	\newblock {\em Prediction, learning, and games}.
	\newblock Cambridge University Press, UK.
	
	\bibitem[Cesari and Colomboni, 2025]{cesari2024trading}
	Cesari, T. and Colomboni, R. (2025).
	\newblock An online learning theory of trading-volume maximization.
	\newblock In {\em The Thirteenth International Conference on Learning
		Representations}.
	
	\bibitem[Colini{-}Baldeschi et~al., 2016]{Colini-Baldeschi16}
	Colini{-}Baldeschi, R., de~Keijzer, B., Leonardi, S., and Turchetta, S. (2016).
	\newblock Approximately efficient double auctions with strong budget balance.
	\newblock In {\em {ACM-SIAM} Symposium on Discrete Algorithms, {SODA}'16},
	pages 1424--1443, USA. {SIAM}.
	
	\bibitem[Colini{-}Baldeschi et~al., 2017]{Colini-Baldeschi17}
	Colini{-}Baldeschi, R., Goldberg, P.~W., de~Keijzer, B., Leonardi, S., and
	Turchetta, S. (2017).
	\newblock Fixed price approximability of the optimal gain from trade.
	\newblock In {\em Web and Internet Economics, {WINE}'17}, volume 10660 of {\em
		Lecture Notes in Computer Science}, pages 146--160, Germany. Springer.
	
	\bibitem[Colini-Baldeschi et~al., 2020]{colini2020approximately}
	Colini-Baldeschi, R., Goldberg, P.~W., Keijzer, B.~d., Leonardi, S.,
	Roughgarden, T., and Turchetta, S. (2020).
	\newblock Approximately efficient two-sided combinatorial auctions.
	\newblock {\em ACM Transactions on Economics and Computation (TEAC)},
	8(1):1--29.
	
	\bibitem[Deng et~al., 2022]{DengMSW21}
	Deng, Y., Mao, J., Sivan, B., and Wang, K. (2022).
	\newblock Approximately efficient bilateral trade.
	\newblock In {\em {STOC}}, pages 718--721, Italy. {ACM}.
	
	\bibitem[D\"{u}tting et~al., 2021]{dutting2021efficient}
	D\"{u}tting, P., Fusco, F., Lazos, P., Leonardi, S., and Reiffenh\"{a}user, R.
	(2021).
	\newblock Efficient two-sided markets with limited information.
	\newblock In {\em Proceedings of the 53rd Annual ACM SIGACT Symposium on Theory
		of Computing}, STOC 2021, page 1452–1465, New York, NY, USA. Association
	for Computing Machinery.
	
	\bibitem[Hazan and Megiddo, 2007]{hazan2007online}
	Hazan, E. and Megiddo, N. (2007).
	\newblock Online learning with prior knowledge.
	\newblock In {\em Learning Theory: 20th Annual Conference on Learning Theory,
		COLT 2007, San Diego, CA, USA; June 13-15, 2007. Proceedings 20}, pages
	499--513. Springer.
	
	\bibitem[Kang et~al., 2022]{kang2022fixed}
	Kang, Z.~Y., Pernice, F., and Vondr{\'a}k, J. (2022).
	\newblock Fixed-price approximations in bilateral trade.
	\newblock In {\em Proceedings of the 2022 Annual ACM-SIAM Symposium on Discrete
		Algorithms (SODA)}, pages 2964--2985, Alexandria, VA, USA. SIAM, Society for
	Industrial and Applied Mathematics.
	
	\bibitem[Kang and Vondr{\'a}k, 2019]{kang2019fixed}
	Kang, Z.~Y. and Vondr{\'a}k, J. (2019).
	\newblock Fixed-price approximations to optimal efficiency in bilateral trade.
	\newblock {\em Available at SSRN 3460336}.
	
	\bibitem[Lucas~Jr, 1989]{lucas1989effects}
	Lucas~Jr, R.~E. (1989).
	\newblock The effects of monetary shocks when prices are set in advance.
	\newblock {\em Unpublished paper. University of Chicago}.
	
	\bibitem[Myerson and Satterthwaite, 1983]{myerson1983efficient}
	Myerson, R.~B. and Satterthwaite, M.~A. (1983).
	\newblock Efficient mechanisms for bilateral trading.
	\newblock {\em Journal of economic theory}, 29(2):265--281.
	
	\bibitem[Slivkins, 2011]{slivkins2011contextual}
	Slivkins, A. (2011).
	\newblock Contextual bandits with similarity information.
	\newblock In {\em Proceedings of the 24th annual Conference On Learning
		Theory}, pages 679--702. JMLR Workshop and Conference Proceedings.
	
	\bibitem[Weill, 2020]{weill2020search}
	Weill, P.-O. (2020).
	\newblock The search theory of over-the-counter markets.
	\newblock {\em Annual Review of Economics}, 12:747--773.
	
	\bibitem[Williams, 1991]{williams1991probability}
	Williams, D. (1991).
	\newblock {\em Probability with martingales}.
	\newblock Cambridge university press, UK.
	
	\bibitem[www.bis.org, 2022]{bis2023}
	www.bis.org (2022).
	\newblock {OTC} derivatives statistics at end-{J}une 2022.
	\newblock {\em Bank for International Settlements}.
	
\end{thebibliography}
\end{document}